\documentclass{article}
\usepackage{iclr2027_conference}
\usepackage{times}
\usepackage{amsmath,amssymb}
\usepackage[margin=1in]{geometry}
\usepackage{booktabs}
\usepackage{arydshln}
\usepackage{graphicx}
\usepackage{subcaption}
\usepackage{multirow}
\usepackage{algorithm,algpseudocode}
\usepackage{amsthm}
\usepackage{soul,ulem}
\usepackage{cancel}
\newtheorem{theorem}{Theorem}

\usepackage{algorithm}
\usepackage{algpseudocode}
\usepackage{wrapfig}
\usepackage{bm}
\usepackage{array}
\usepackage{graphicx}
\usepackage{booktabs}
\usepackage{threeparttable}
\usepackage{amssymb}
\usepackage{pifont}

\algtext*{EndFor}
\algtext*{EndIf}

\usepackage{pifont}

\newcommand{\cmark}{\ding{51}}
\newcommand{\xmark}{\ding{55}}

\algrenewcommand\algorithmicrequire{\textbf{Input:}}
\algrenewcommand\algorithmicensure{\textbf{Output:}}

\usepackage{amsmath,amsfonts,bm}

\def\eqref#1{equation~\ref{#1}}

\def\1{\bm{1}}

\def\ra{{\textnormal{a}}}

\def\rx{{\textnormal{x}}}

\def\rva{{\mathbf{a}}}

\def\erva{{\textnormal{a}}}

\def\ervx{{\textnormal{x}}}

\def\rmA{{\mathbf{A}}}

\def\vmu{{\bm{\mu}}}
\def\vtheta{{\bm{\theta}}}
\def\va{{\bm{a}}}

\def\ve{{\bm{e}}}

\def\vx{{\bm{x}}}

\def\eva{{a}}

\def\mA{{\bm{A}}}

\def\mH{{\bm{H}}}
\def\mI{{\bm{I}}}
\def\mJ{{\bm{J}}}

\def\mX{{\bm{X}}}

\def\mSigma{{\bm{\Sigma}}}

\DeclareMathAlphabet{\mathsfit}{\encodingdefault}{\sfdefault}{m}{sl}
\SetMathAlphabet{\mathsfit}{bold}{\encodingdefault}{\sfdefault}{bx}{n}
\newcommand{\tens}[1]{\bm{\mathsfit{#1}}}
\def\tA{{\tens{A}}}

\def\tX{{\tens{X}}}

\def\gG{{\mathcal{G}}}

\def\sA{{\mathbb{A}}}
\def\sB{{\mathbb{B}}}

\def\sS{{\mathbb{S}}}

\def\emA{{A}}

\newcommand{\etens}[1]{\mathsfit{#1}}

\def\etA{{\etens{A}}}

\newcommand{\E}{\mathbb{E}}

\newcommand{\R}{\mathbb{R}}

\newcommand{\KL}{D_{\mathrm{KL}}}
\newcommand{\Var}{\mathrm{Var}}

\newcommand{\Cov}{\mathrm{Cov}}

\newcommand{\normltwo}{L^2}
\newcommand{\normlp}{L^p}

\newcommand{\parents}{Pa}

\usepackage{xspace}

\newcommand{\method}{$\mathtt{RIDE}$\xspace}
\newcommand{\methodonehop}{{$\mathtt{RIDE}^{\text{(1)}}$}\xspace}

\newcommand{\methodite}{{$\mathtt{RIDE}^{\text{(2)}}$}\xspace}

\newcommand{\conditar}{{conDitar}\xspace}

\newcommand{\ftconditar}{{conDitar-a}\xspace}

\newcommand{\ftconditarqed}{{conDitar-a}({QED})\xspace}

\newcommand{\ipdiff}{{IPDiff}\xspace}

\newcommand{\ftipdiff}{{IPDiff-a}\xspace}

\newcommand{\ftipdiffqed}{{IPDiff-a}({QED})\xspace}

\newcommand{\diffhopp}{{DiffHopp}\xspace}

\newcommand{\shepherd}{{ShEPhERD}\xspace}

\newcommand{\conditardev}{{conDitar-dev}\xspace}

\newcommand{\hopone}{{one hop}\xspace}
\newcommand{\hoptwo}{{two hops}\xspace}

\newcommand{\simtwod}{Sim$_{\text{2D}}$\xspace}
\newcommand{\simthreed}{Sim$_{\text{3D}}$\xspace}

\newcommand{\connectivity}{Conn.(\%)\xspace}

\usepackage{hyperref}
\usepackage{url}

\iclrfinalcopy

\title{\method: Reference-Anchored Inference-Time Diffusion Editing for Scaffold Hopping}

\author{Ruoxi Gao$^1$, Frazier N. Baker$^1$, Trieu Nguyen$^1$, and Xia Ning$^{1,2,3,4,*}$ \\
$^1$ Department of Computer Science and Engineering, The Ohio State University \\
$^2$ Department of Biomedical Informatics, The Ohio State University \\
$^3$ Translational Data Analytics Institute, The Ohio State University \\
$^4$ Division of Medicinal Chemistry and Pharmacognosy, The Ohio State University \\
$^*$Corresponding author: \url{ning.104@osu.edu}}

\begin{document}

\maketitle

\begin{abstract}

Scaffold hopping is a critical task in drug discovery, which seeks to discover new, structurally distinct molecules that share key functional groups and similar 3D shape 
with a reference binding ligand. 
Existing diffusion-based scaffold hopping methods formulate the problem 
as conditional generation of scaffolds given the functional groups. 
However, they lack a principled mechanism to 
jointly enforce 2D structural novelty and preserve the 3D shape of the reference 
ligand. 
Here, we introduce \method, 
a Reference-anchored Inference-time Diffusion Editing framework for scaffold hopping. 
\method recovers the reference diffusion noise trajectory conditioned on the binding pocket and functional groups, selects an optimal trajectory segment for editing via noise perturbation,
and conducts a value-guided scaffold sampling to generate new scaffolds. 
Extensive experimental results demonstrate that, 
compared to baselines, \method consistently generates scaffolds with lower 2D similarity and higher 3D similarity to the reference,
with an average improvements of 11.7\% and 7.3\%, respectively.
Further analysis reveals that \method can accommodate various reward functions,
and can preserve 3D similarity even when this is not explicitly included in the reward.
Two case studies illustrate \method's ability to generate distinct scaffolds with different structures and properties,
and its ability to introduce substantial 2D variation while maintaining very high 3D similarity.
\method is publicly available at \url{https://anonymous.4open.science/r/RIDE-C8A0}.

\end{abstract}

\section{Introduction}

Scaffold hopping is a critical task in drug discovery, which seeks to discover new, structurally distinct molecules that share key functional groups and similar 3D shape 
with a reference binding ligand~\citep{bohm2004scaffold}.
Effective scaffold hopping increases the success rate of drug discovery by exploring
a variety of molecule structures with similar binding but different properties, synthesizability,
and patentability~\citep{bohm2004scaffold}.
Recently, generative AI methods have opened new avenues for scaffold hopping~\citep{zheng2021deep,torge_diffhopp, yoo2024turbohopp, yang2026diffusion}.
These methods provide the ability to sample new molecular structures from distributions learned from vast datasets,
enabling the exploration of chemical spaces beyond those familiar to domain experts and 
expanding opportunities to discover novel molecular scaffolds.

Among generative AI methods,
diffusion models are particularly promising for scaffold hopping,
as their iterative denoising process enables flexible and fine-grained control over generation.
Several efforts~\citep{torge_diffhopp, adams_shepherd,yang2026diffusion} have applied
diffusion to scaffold hopping,
formulating the problem as the conditional generation of scaffolds given the functional groups.
While these methods can generate new scaffolds, 
they rely on sampling stochasticity to diversify their generated results, and  
lack a principled mechanism to promote structural novelty, leaving the extent of structural divergence from the original scaffold largely to chance.
While scaffold hopping based on diffusion models for structure-based drug design (SBDD)~\citep{guan20233d,guan_decompdiff,huang_ipdiff} can leverage protein-ligand binding information, 
it lacks explicit guarantees of substantial 2D structural novelty relative to known ligands.
Conversely, approaches based on diffusion for ligand-based drug design (LBDD)~\citep{chen2025generating,adams_shepherd} can explore structurally distinct scaffolds 
but lack explicit guidance from protein-ligand binding.

\begin{figure}[t]
\centering
\includegraphics[width=0.95\textwidth]{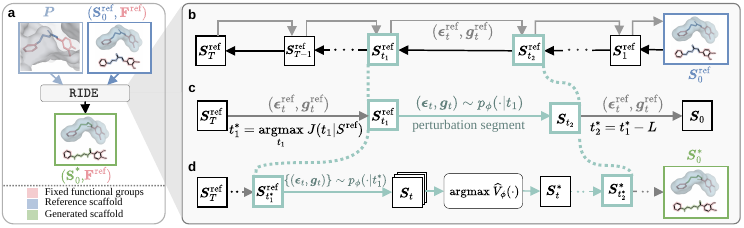}
\caption{Overview of \method. \textbf{a}, \method generates new scaffolds based on a reference scaffold, fixed functional groups, and protein pocket; \textbf{b}, \method recovers the noise trajectory for the reference ligand from a diffusion-based molecule generation model adapted to scaffold hopping; \textbf{c}, \method identifies a reference-optimal noise trajectory segment for editing; and
\textbf{d}, \method performs a value-guided transition within the reference-optimal trajectory segment for a controlled editing of the noise trajectory, thus sampling a set of new scaffolds.
}
\label{fig:overview}
\vspace{-5pt}
\end{figure}

We introduce \method, a \underline{R}eference-anchored \underline{I}nference-time \underline{D}iffusion \underline{E}diting framework for scaffold hopping. 
\method reformulates scaffold hopping as a diffusion editing process that
transforms the reference scaffold into new scaffolds.
As shown in Figure~\ref{fig:overview}, it operates in three steps:
\textbf{(1)} recovery of the reference diffusion noise trajectory conditioned on the binding pocket and functional groups,
\textbf{(2)} selection of an optimal trajectory segment for editing via noise perturbation,
and \textbf{(3)} value-guided scaffold sampling to generate new scaffolds. 
{\method} leverages a reward function measuring 2D structural novelty and 3D shape 
similarity to guide noise perturbation and scaffold sampling, 
jointly enforcing 2D structural novelty and preserving the 3D shape of the reference ligand.
As a framework for inference-time editing, 
\method can adapt pre-trained diffusion-based molecular generative models 
for controllable scaffold hopping without retraining.

We rigorously evaluate \method with state-of-the-art baseline models on 60 scaffold hopping tasks for protein-ligand complexes related to major human diseases.
Compared to baselines, \method consistently generates scaffolds with lower 2D similarity and higher 3D similarity to the reference,
with an average improvements of 11.7\% and 7.3\%, respectively.
Further analysis reveals that \method can accommodate various reward functions,
and can preserve 3D similarity even when this is not explicitly included in the reward.
\method can be applied iteratively to explore even more structurally distinct scaffolds, 
where a \method-generated molecule is used as a new ``reference'' to \method for 
further scaffold hopping.
Experiments reveal a trade-off in multiple hops between progressively improving 
the similarity profile and preserving binding. 
Additionally, two case studies illustrate \method's ability to generate distinct scaffolds with different structures and properties,
and its ability to introduce substantial 2D variation while maintaining very high 3D similarity.
\method is publicly available at \url{https://anonymous.4open.science/r/RIDE-C8A0}.

\vspace{-7pt}
\section{Related work}
\vspace{-7pt}

\paragraph{Generative models for Scaffold Hopping}

Recently, generative models have shown substantial promise in SBDD, LBDD,  and scaffold hopping~\citep{zheng2021deep,guan20233d,torge_diffhopp}.
SBDD methods generate molecules tailored to specific protein pockets without relying on a reference ligand~\citep{guan20233d,guan_decompdiff,huang_ipdiff,gu2024aligning,yangdecode,gao_conditardev}.
For example, 
\ipdiff~\citep{huang_ipdiff} incorporates learned pocket-ligand interaction embeddings into the diffusion processes to generate molecules with strong binding affinity.
Similarly,
\conditardev~\citep{gao_conditardev} trains its diffusion model to leverage pretrained pocket representations and supports inference-time trajectory optimization to optimize various molecule properties.
In contrast to SBDD, LBDD methods leverage known binding ligands as references to generate molecules with high 3D shape similarity to the reference~\citep{chen2025generating,adams_shepherd,li2026novo}.
For example, \shepherd~\citep{adams_shepherd} jointly learns molecular structures with 3D shape, electrostatics, and pharmacophore to generate molecules that resemble the reference ligand with respect to these features.
Generative models for scaffold hopping can be formulated under either SBDD or LBDD through inpainting~\citep{schneuing2024structure,li2026novo} or additional conditioning on functional groups~\citep{torge_diffhopp,yoo2024turbohopp}.
However, structure-based scaffold hopping and ligand-based scaffold hopping emphasize either binding affinity or the 3D similarity to the reference, respectively.
Moreover, current formulations do not explicitly optimize 2D similarity to the reference scaffold.
Our work instead incorporates the reference scaffold into the generation process of structure-based models,
providing a unified framework for achieving low 2D similarity and high 3D similarity while maintaining strong binding affinity.

\vspace{-5pt}
\section{Preliminaries}

\subsection{Inversion-Based Diffusion Methods}

Inversion-based diffusion methods aim to invert a reference sample into its latent diffusion
trajectory and edit the recovered trajectory for controlled generation~\citep{kim2022diffusionclip,mokady2023null,hertz2022prompt}. 
By editing the recovered trajectory, these methods can generate samples that preserve reference-specific features and introduce desired changes.
These methods build upon the reverse sampling process of diffusion models.
A common approach to recovering the reference state trajectory is DDIM inversion~\citep{song_ddim,wallace2023edict}.
Starting from a reference sample $\bm{s}_0^{\mathrm{ref}}$,
DDIM inversion maps it back to the corresponding intermediate states $\{\bm{s}_t^{\mathrm{ref}}\}_{t=1}^{T}$, where \(T\) denotes the total number of diffusion timesteps.
The inversion process is represented as
\vspace{-5pt}
\begin{equation}
    \label{eq:continuous_inversion}
    \bm{s}_t^{\mathrm{ref}}=\sqrt{\bar{\alpha}_t}
    ((\bm{s}_{t-1}^{\mathrm{ref}}-\sqrt{1-\bar{\alpha}_{t-1}} \bm{\epsilon}_\theta(\bm{s}_{t-1}^{\mathrm{ref}}, t-1))/\sqrt{\bar{\alpha}_{t-1}})
    +
    \sqrt{1-\bar{\alpha}_t} \bm{\epsilon}_\theta(\bm{s}_{t-1}^{\mathrm{ref}}, t-1) ,
\end{equation}
where $\bm{\epsilon}_\theta(\bm{s}_{t},t)$ is the noise predictor learned during diffusion training, 
$\bar{\alpha}_t\in (0, 1)$ 
is a time-dependent cumulative coefficient determined by the predefined noise schedule.
Under DDPM sampling~\citep{ho_ddpm},
each transition $\bm{s}_{t}^{\mathrm{ref}}\rightarrow \bm{s}_{t-1}^{\mathrm{ref}}$ along the recovered reference state trajectory is associated with a specific sampling noise $\bm{\epsilon}_t^{\mathrm{ref}}$. 
Therefore, the reference state trajectory can be equivalently represented by a corresponding noise sequence $\{\bm{\epsilon}_t^{\mathrm{ref}}\}_{t=1}^T$,
with each $\bm{\epsilon}_t^{\mathrm{ref}}$ recovered as follows~\citep{huberman2024edit},
\vspace{-3pt}
\begin{equation}
    \bm{\epsilon}_t^{\mathrm{ref}}=(\bm{s}^{\mathrm{ref}}_{t-1}-\bm{\mu}_\theta(\bm{s}^{\mathrm{ref}}_t, t))/\sigma_t ,
    \label{eq:continuous_noise_inversion}
\end{equation}
where $\bm{\mu}_\theta(\bm{s}_t, t)=\frac{1}{\sqrt{\alpha_t}}(\bm{s}_t-\frac{\beta_t}{\sqrt{1-\bar{\alpha}_t}} \bm{\epsilon}_\theta(\bm{s}_t, t))$ is the parameterized mean of the reverse transition distribution, $\beta_t$ is the noise variance schedule at timestep $t$,
and $\sigma_t$ is the predefined standard deviation of the transition distribution. 

Although the inversion-based methods were originally proposed for Gaussian diffusion,
the inversion can also be extended to categorical diffusion,
where each state ${s}_t^{\mathrm{ref}}$ in the reference trajectory is categorical. 
The inversion process can be constructed as follows~\citep{he2026dice},
\vspace{-10pt}
\begin{equation}
s_t^{\mathrm{ref}}
\sim
q(
s_t \mid s_{t-1}^{\mathrm{ref}}
),
\label{eq:discrete_inversion}
\end{equation}
where $q(s_t \mid s_{t-1})=\operatorname{Cat}(
\alpha^{c}_{t}
\bm{e}_{s_{t-1}})+
(1-\alpha^{c}_{t})\frac{1}{d}\mathbf{1})$
denotes the categorical forward transition~\citep{hoogeboom2021argmax}, $\bm{e}_s$ is the one-hot vector corresponding to category $s$,
$d$ is the number of discrete categories, and $\alpha_t^{c}\in (0,1)$ is predefined noise schedule at timestep $t$.
Given the recovered categorical reference trajectory $\{{s}_t^{\mathrm{ref}}\}_{t=1}^T$, 
the 
corresponding noise sequence $\{\bm{g}_t^{\mathrm{ref}}\}_{t=1}^T$ can be determined as follows,
\begin{equation}
    g_{t, k}^{\mathrm{ref}}=  [\max\nolimits_{j \neq s_{t-1}^{\mathrm{ref}}} \log p_\theta(s_{t-1}=j \mid s_t^{\mathrm{ref}})-\log p_\theta(s_{t-1}=s_{t-1}^{\mathrm{ref}} \mid s_t^{\mathrm{ref}})+\delta]_{+}, \text{if }k=s_{t-1}^{\mathrm{ref}} \text{; otherwise, }0,
    \label{eq:discrete_noise_inversion}
\end{equation}
where $p_\theta(s_{t-1} \mid s_t)$ denotes the one-step reverse categorical transition,
$\delta$ is a small positive margin to ensure that the category can be stricly selected, 
and
$[\cdot]_+=\max(\cdot,0)$ denotes the positive-part operator.

\subsection{Value-Based Methods for Inference-Time Optimization in Diffusion}

Inference-time optimization methods for diffusion models focus on modifying their
reverse process to improve generated samples with respect to specific objectives~\citep{chung2022diffusion,tang2024inference,kim2025test}.
However, most existing approaches rely on gradients~\citep{guo2026training,yang2024diffusion}, which are not often available for objectives in molecular generation~\citep{shen2025chemistry,tan2025fast}.
Value-based methods provide an alternative for non-differentiable objectives 
by estimating the expected future reward of an intermediate state $\bm{s}_t$ 
with the value function defined as $V(\bm{s}_t)=\mathbb{E}_{\bm{s}_0\sim p_\theta(\cdot\mid \bm{s}_t)}
\left[
\mathcal{R}(\bm{s}_0)
\right]$~\citep{li2024derivative,uehara2025inference},
where $\mathcal{R}(\cdot)$ is the reward function defined on the final state $\bm{s}_0$.

Recent efforts~\citep{kim2026lookahead} reformulate the value estimates using the samples generated from 
$p_\theta(\bm{s_0})$ as follows to bypass potential computational limitations in using $p_\theta(\cdot\mid \bm{s}_t)$~\citep{li2024derivative,jain2025diffusion}:
\begin{equation}
V(\bm{s}_t)=\mathbb{E}_{p_\theta(\bm{s}_0)}\left[q(\bm{s}_t \mid \bm{s}_0) \mathcal{R}(\bm{s}_0)\right]/\mathbb{E}_{p_\theta(\bm{s}_0)}\left[q(\bm{s}_t \mid \bm{s}_0)\right].
\label{eq:value_est_lidar}
\end{equation}
This reformulation enables $V(\bm{s}_t)$ to be estimated in two steps: (1) perform 
$\bm{s}_t$-independent lookahead sampling from the noise prior $q(\bm{s}_T)$ to obtain $\bm{s}_0$ and $\mathcal{R}(\bm{s}_0)$, and (2) compute $V(\bm{s}_t)$ in closed form by connecting $\bm{s}_t$ to the lookahead samples $\bm{s}_0$ through the forward transition $q(\bm{s}_t\mid \bm{s}_0)$.

\section{Methods}

\subsection{Problem Definition and Notations}
\label{sec:formulation}

A ligand can be decomposed into a scaffold and a set of functional groups. 
The functional groups correspond to atoms involved in key interactions with the ligand's binding pocket,
and the scaffold atoms define the molecular backbone connecting them.
Scaffold hopping seeks to identify ligands that preserve the 3D shape and binding pattern of a known reference ligand (or ``\textit{reference}") and that are also structurally distinct from it in 2D.
Accordingly, 
we represent a ligand as $\bm{M}=(\bm{S},\bm{F})$, where $\bm{S}$ denotes the scaffold and $\bm{F}$ denotes the functional groups,
and represent the protein pocket as $\bm{P}$. 
$\bm{S}$, $\bm{F}$, and $\bm{P}$ are all represented as sets of atoms involved, 
with each atom $\bm{a}$ 
characterized by a continuous position $\bm{x}\in\mathbb{R}^3$ and a categorical atom type $v \in \{1,\ldots,d\}$.

Given a pocket $\bm{P}$ and a reference ligand 
$\bm{M}^{\mathrm{ref}}=(\bm{S}^{\mathrm{ref}},\bm{F}^{\mathrm{ref}})$
bound to $\bm{P}$, 
scaffold hopping can be formulated as generating 
a new ligand
$\bm{M}=(\bm{S},\bm{F}^{\mathrm{ref}})$
by replacing the reference scaffold $\bm{S}^{\mathrm{ref}}$ with a different,
generated scaffold $\bm{S}$.
Existing methods~\citep{torge_diffhopp, yoo2024turbohopp} typically formulate this problem as conditional scaffold generation 
using a diffusion model, that is, $\bm{S} \sim p_\theta(\cdot \mid \bm{P}, \bm{F}^{\mathrm{ref}})$, parameterized by $\theta$,
where the reference scaffold $\bm{S}^{\mathrm{ref}}$ is not involved in the generation process.
As a result, the diffusion model cannot explicitly exploit the structural features of the reference scaffold to tailor generation to the desired 2D and 3D similarity profile.

Here, we propose \method (Figure~\ref{fig:overview}), 
a novel reference-anchored editing framework that reformulates scaffold hopping as a principled transformation from the reference scaffold $\bm{S}^{\mathrm{ref}}$ to a set of new scaffolds $\{\bm{S}\}$.
This formulation enables critical reference-specific features, such as the 3D shape and binding pose, 
to be preserved while introducing desired 2D changes to the scaffolds.
To implement this formulation,
\method first adapts
state-of-the-art
SBDD models to scaffold hopping to construct the base model $p_\theta(\bm{S} \mid \bm{P}, \bm{F}^{\mathrm{ref}})$, with details described in Appendix~\ref{app:finetune_SBDD},
and propose an inference-time method that modifies its sampling process to obtain a new distribution $\tilde{p}_{\theta}(\bm{S}\mid \bm{P},\bm{F}^{\mathrm{ref}};\bm{S}^{\mathrm{ref}})$ anchored to 
$\bm{S}^{\mathrm{ref}}$ (Section~\ref{sec:reference_edit}).
Thus, the new distribution can be identified by solving the following new optimization problem:
\begin{equation}
\max\nolimits_{\tilde{p}_{\theta}} 
\mathbb{E}_{\bm{S}\sim\tilde{p}_{\theta}(\cdot\mid \bm{P},\bm{F}^{\mathrm{ref}};\bm{S}^{\mathrm{ref}})}
\left[
\mathcal{R}(\bm{S},\bm{S}^{\mathrm{ref}})
\right], 
\label{single_hop}
\end{equation}
where $\mathcal{R}()$ is the reward function defined as follows:
\begin{equation}
\label{eqn:reward}
\mathcal{R}(\cdot,\bm{S}^{\mathrm{ref}}) = 
    \lambda(1-\operatorname{Sim}_{2 \mathrm{D}}(\cdot, \bm{S}^{\mathrm{ref}}))+(1-\lambda) \operatorname{Sim}_{3 \mathrm{D}}(\cdot, \bm{S}^{\mathrm{ref}}),
\end{equation}
which considers 2D similarity $\operatorname{Sim}_{2 \mathrm{D}}$ and 
3D similarity $\operatorname{Sim}_{3 \mathrm{D}}$ to $\bm{S}^{\mathrm{ref}}$; 
and $\lambda\in [0,1]$ balances the two similarities.
The reward is computed only for structurally valid scaffolds and is set to $0$ for invalid ones.

\subsection{Reference-Anchored Scaffold Generation via Noise Editing}
\label{sec:reference_edit}

Our goal is to construct the generation distribution $\tilde{p}_\theta(\bm{S} \mid \bm{P}, \bm{F}^{\mathrm{ref}} ; \bm{S}^{\mathrm{ref}})$ from the base scaffold generation distribution $p_\theta(\bm{S} \mid \bm{P}, \bm{F}^{\mathrm{ref}})$.
To address the competing objectives of
low 2D similarity and high 3D similarity with 
$\bm{S}^{\mathrm{ref}}$
desired by the generated scaffolds, 
we anchor scaffold generation to $\bm{S}^{\mathrm{ref}}$ 
retain 3D shape, and 
apply a new value-guided sampling to encourage simultaneous 2D deviation. 
The algorithm is summarized in Appendix~\ref{appx:algorithm}.

\subsubsection{Reference Noise Trajectory Recovery}
\label{sec:traj_recover}
To establish a reference anchor for controlled scaffold
generation during inference time,
\method first recovers
the state and noise trajectory reproducing $\bm{S}^{{\mathrm{ref}}}$  
under the base diffusion model $p_\theta(\bm{S} \mid \bm{P}, \bm{F}^{\mathrm{ref}})$, as illustrated in Figure~{\ref{fig:overview}}b.
We represent 
$\bm{S}^{\mathrm{ref}}$ as 
$(\bm{X}^{{\mathrm{ref}}}, \bm{V}^{{\mathrm{ref}}})$, where $\bm{X}^{{\mathrm{ref}}}$ and $\bm{V}^{{\mathrm{ref}}}$ denote the scaffold atom positions and atom types, respectively.
As the base model employs Gaussian diffusion for atom positions and categorical diffusion for atom types,
we recover the reference position trajectory
$\{\bm{X}_t^{{\mathrm{ref}}}\}_{t=0}^T$ and 
reference atom-type trajectory
$\{\bm{V}_t^{{\mathrm{ref}}}\}_{t=0}^T$
  using Eq.~\ref{eq:continuous_inversion} and Eq.~\ref{eq:discrete_inversion}, respectively, 
with $\bm{X}_0^{\mathrm{ref}}=\bm{X}^{\mathrm{ref}}$ and
$\bm{V}_0^{\mathrm{ref}}=\bm{V}^{\mathrm{ref}}$.
The corresponding position noise trajectory $\{\bm{\epsilon}_t^{{\mathrm{ref}}}\}_{t=1}^T$ and atom-type noise trajectory $\{\bm{g}_t^{{\mathrm{ref}}}\}_{t=1}^T$ are then 
derived
from the reference state trajectories using Eq.~\ref{eq:continuous_noise_inversion} and Eq.~\ref{eq:discrete_noise_inversion}, respectively.

Here,
we denote the reverse transition with a specific noise sequence from timestep $t_1$ to $t_2$ ($T \ge t_1 > t_2 \ge 0$) as 
\begin{equation}
\mathcal{T}_{\theta,t_1:t_2}(
\bm{S}_{t_1},
\{(\bm{\epsilon}_\tau,\bm{g}_\tau)\}_{\tau=t_2+1}^{t_1}
; \bm{P}, \bm{F}^{\mathrm{ref}}),
\end{equation}
and the transition using the reference noise as
\begin{equation}
\mathcal{T}_{\theta,t_1:t_2}^{\mathrm{ref}}(\bm{S}_{t_1}; \bm{P}, \bm{F}^{\mathrm{ref}})
=
\mathcal{T}_{\theta,t_1:t_2}(
\bm{S}_{t_1},
\{(\bm{\epsilon}_\tau^{\mathrm{ref}},\bm{g}_\tau^{\mathrm{ref}})\}_{\tau=t_2+1}^{t_1}
; \bm{P}, \bm{F}^{\mathrm{ref}}).
\end{equation}
Thus, we have
\vspace{-10pt}
\begin{equation}
\bm{S}_t^{\mathrm{ref}}=\mathcal{T}^{\mathrm{ref}}_{\theta, T: t}(\bm{S}_T^{\mathrm{ref}}; \bm{P}, \bm{F}^{\mathrm{ref}}).
\end{equation}
For notational simplicity, we omit $\bm{P}$ and $\bm{F}^{\mathrm{ref}}$ in $\mathcal{T}_\theta$ below.

\subsubsection{Reference Trajectory Editing via Noise Perturbation}
\label{sec:traj_perturb}

\method explores alternative scaffolds similar to
the reference via reference trajectory editing, 
and implement this as noise perturbation over a segment of the reference noise trajectory 
between $t_1$ and $t_2$, with $t_2 = \max(t_1 - L, 0)$, where $L>0$ is the segment length. 
Specifically, we design a perturbation distribution 
$p_{\phi}(\bm{\epsilon}_t,\bm{g}_t \mid t_1;\bm{P}, \bm{F}^{\mathrm{ref}}, \{(\bm{\epsilon}_t^{\mathrm{ref}},\bm{g}_t^{\mathrm{ref}})\}_{t=t_2+1}^{t_1},L)$ 
over the atom position noise 
$\bm{\epsilon}_t$ and type noise $\bm{g}_t$ as follows 
($\bm{P}$, $\bm{F}^{\mathrm{ref}}$, $\{(\bm{\epsilon}_t^{\mathrm{ref}},\bm{g}_t^{\mathrm{ref}})\}_{t=t_2+1}^{t_1}$ and $L$ are omitted below for simplicity):
\begin{equation}
    p_{\phi}
    (\bm{\epsilon}_t,\bm{g}_t \mid t_1)
    =
    p_{\phi_{{\epsilon}}}(\bm{\epsilon}_t \mid t_1)
    p_{\phi_{{g}}}(\bm{g}_t \mid t_1),
    \label{eq:noise_perturb}
\end{equation}
where
\vspace{-13pt}
\begin{equation}
p_{\phi_{\epsilon}}(\bm{\epsilon}_t \mid t_1)
=
\begin{cases}
\mathcal{N}(\bm{\epsilon}_t;\bm{0},\bm{I}),
& t_1 \ge t > t_2\\
\delta(\bm{\epsilon}_t-\bm{\epsilon}_t^{\mathrm{ref}}),
& \text{otherwise},
\end{cases}
\qquad
p_{\phi_g}(\bm{g}_t \mid t_1)
=
\begin{cases}
\operatorname{Gumbel}(\bm{g}_t;0,1),
&t_1 \ge t > t_2,\\
\delta(\bm{g}_t-\bm{g}_t^{\mathrm{ref}}),
& \text{otherwise},
\end{cases}
\end{equation}
$\delta(\cdot)$ is the Dirac delta function. 
That is, \method explores different trajectories by editing the reference noise $(\bm{\epsilon}^{\mathrm{ref}}_t, \bm{g}_t^{\mathrm{ref}})$ between $t_1$ and $t_2$,
 replacing it with sampled noise from $(\mathcal{N}(\bm{0},\bm{I}), \operatorname{Gumbel}(0,1))$, 
with the remaining noise intact.
Unlike existing inversion-based methods that directly modify recovered intermediate states, 
we perturb the trajectories by resampling the diffusion noise at selected timesteps. 
This avoids the difficulty of designing appropriate state-space edits and reduces the risk of moving the generation 
too far from the diffusion manifold.
Meanwhile, retaining a majority of the original reference noise trajectories outside 
$t_1$ and $t_2$ provides a strong mechanism for preserving 3D shape and binding pose enabled by $p_{\theta}$.

\paragraph{Reference-Optimal Perturbation Segment Selection}

The optimal trajectory segment to edit
depends on the reference ligand. Thus, 
\method adopts a reference-specific strategy to optimally select 
a noise trajectory segment, illustrated in Figure~\ref{fig:overview}c.
\method first discretizes the search space of $t_1$ as $t_1 \in\{ \tau_n = \frac{n T}{N}: n=n_1, \ldots, n_2\}$, 
where $N>0$ denotes the number of uniformly sampled timesteps along the full diffusion trajectory, 
and $n_1$ and $n_2$ specify the allowed range for $t_1$. 
For each candidate 
$t_1$, we generate $K$ Monte Carlo samples following Eq.~\ref{eq:noise_perturb} 
and define a score $J$ to quantify $t_1$ over 
the $K$ samples as follows:\vspace{-10pt}
\begin{equation}
    J(t_1\mid \bm{S}^{\mathrm{ref}})=\frac{1}{K} \sum\nolimits_{k=1}^K\left[
    \mathcal{R}(\bm{S}^{(k)}(t_1), \bm{S}^{\mathrm{ref}})
        \right],
    \label{eq:timestep_score}
\end{equation}
where
\vspace{-8pt}
\begin{equation}
    \bm{S}^{(k)}(t_1)=\mathcal{T}_{\theta, t_1: 0}(\bm{S}_{t_1}^{{\mathrm{ref}}},\{(\bm{\epsilon}_t^{(k)}, \bm{g}_t^{(k)})\}_{t=1}^{t_1}), 
    \quad\{(\bm{\epsilon}_t^{(k)}, \bm{g}_t^{(k)})\} \sim p_{\phi}(\cdot \mid t_1), 
\end{equation}
and $\mathcal{R}()$ is the reward function (Eq.~\ref{eqn:reward}). 
Given that the recovered reference trajectory already provides strong preservation of the reference 3D shape,
we prioritize reducing 2D similarity when identifying the optimal $t_1$.
Thus, we select the optimal $t_1$ as follows, 
\begin{equation}
    t_1^\ast(\bm{S}^{\mathrm{ref}})=\arg \max _{t_1 \in\{{\tau_n}\}} J(t_1\mid \bm{S}^{\mathrm{ref}})
    \text{, and thus, }{t_2^\ast = \max(t_1^\ast - L, 0)}, 
\end{equation}
where $L$ is a hyperparameter determining the length of the perturbation segment.
That is, 
the optimal $t_1^\ast$ ($\bm{S}^{\mathrm{ref}}$ is omitted for simplicity) identifies a segment that produces promising scaffolds under $p_\phi$,  
providing a favorable starting point for subsequent optimization.

\subsubsection{Value-Guided Scaffold Sampling}        
\label{sec:value_estimate}
With the optimal perturbation segment starting from $t_1^\ast$ , \method estimates the value of each 
intermediate scaffold state $\bm{S}_t$ ($t_1^\ast>t>t_2^\ast$) 
sampled under $p_{\phi}(\cdot\mid t_1^\ast)$ 
(Eq.~{\ref{eq:noise_perturb}})
to guide state selection within the segment, as illustrated in Figure~\ref{fig:overview}d.
\method reformulates the value estimator in Eq.~\ref{eq:value_est_lidar} to adapt it to a new reference-anchored sampling, 
using 
$\bm{S}^\mathrm{ref}_{t_1^\ast}$ and samples from the segment endpoint $t_2^\ast$. 

\begin{theorem}
Given a selected perturbation starting timestep $t_1^\ast$ and the corresponding reference state
$\bm{S}_{t_1^\ast}^{\mathrm{ref}}$,
the value of any intermediate state $\bm{S}_t$, where $t_1^\ast>t>t_2^\ast$, can be expressed as
\begin{equation}
V_{\phi}(
\bm{S}_t
\mid
\bm{S}_{t_1^\ast}^{\mathrm{ref}}
)
=
\frac{
\mathbb{E}_{
p_{\theta,\phi}(
\bm{S}_{t_2^\ast}
\mid
\bm{S}_{t_1^\ast}^{\mathrm{ref}}
)
}
\left[
w(
\bm{S}_{t_2^\ast};
\bm{S}_t,
\bm{S}_{t_1^\ast}^{\mathrm{ref}}
)
\mathcal{R}(
\mathcal{T}_{\theta,t_2^\ast:0}^{\mathrm{ref}}
(\bm{S}_{t_2^\ast}),
\bm{S}^{\mathrm{ref}}
)
\right]
}{
\mathbb{E}_{
p_{\theta,\phi}(
\bm{S}_{t_2^\ast}
\mid
\bm{S}_{t_1^\ast}^{\mathrm{ref}}
)
}
\left[
w(
\bm{S}_{t_2^\ast};
\bm{S}_t,
\bm{S}_{t_1^\ast}^{\mathrm{ref}}
)
\right]
},
\label{eq:value_estimator}
\end{equation}
where
\vspace{-10pt}
\begin{equation}
w(
\bm{S}_{t_2^\ast};
\bm{S}_t,
\bm{S}_{t_1^\ast}^{\mathrm{ref}}
)
=
q(
\bm{S}_t
\mid
\bm{S}_{t_2^\ast}
)
/q(
\bm{S}_{t_1^\ast}^{\mathrm{ref}}
\mid
\bm{S}_{t_2^\ast}
),
\end{equation}
$\mathcal{R}(\cdot, \bm{S}^{\mathrm{ref}})$ is the reward function 
(Eq.~\ref{eqn:reward}),
and {$p_{\theta,\phi}(\bm{S}_{t_2^\ast}\mid \bm{S}_{t_1^\ast}^{\mathrm{ref}})$} (conditioning on $\bm{P}$ and $\bm{F}^{\mathrm{ref}}$ is omitted for simplicity)
denotes the transition distribution from timestep $t_1^\ast$ to $t_2^\ast$ under $p_\phi(\cdot\mid t_1^\ast)$ and diffusion parameters $\theta$.
\end{theorem}
The conditioning on $\bm{P}$, $\bm{F}^{\mathrm{ref}}$ and $\bm{S}^{\mathrm{ref}}$ is also omitted in $V_\phi(\bm{S}_t\mid \bm{S}_{t_1^\ast}^{\mathrm{ref}}
)$ to align with $p_{\theta,\phi}(\bm{S}_{t_2^\ast}\mid \bm{S}_{t_1^\ast}^{\mathrm{ref}})$.
We provide the proof in Appendix~\ref{appx:theorem1}. 
In practice, we approximate the distribution $p_{\theta,\phi}(\bm{S}_{t_2^\ast}\mid \bm{S}_{t_1^{*}}^{\mathrm{ref}})$ 
using samples obtained from it, which yields a Monte Carlo estimator
$\widehat V_{\phi}(
    \bm{S}_t\mid \bm{S}_{t_1^\ast}^{\mathrm{ref}}
    )$, 
described in Appendix~\ref{appx:theorem2-0}.

\paragraph{Value-Guided Transition}
\label{sec:guided_sampling}

Within the perturbation segment between $t_1^\ast$ and $t_2^\ast$, \method performs 
value-guided greedy search for optimal scaffold state candidates. 
Specifically, 
at selected timesteps,
given the current state $\bm{S}_t$ {($t_1^\ast > t-1 > t_2^\ast$)}, 
\method samples noises 
from $p_{\phi}(\cdot\mid t_1^\ast)$ (Eq.~\ref{eq:noise_perturb}), and identifies the optimal 
noise based on the estimated value (Eq.~\ref{eq:value_est_mc}) as follows,
\begin{equation}
(\bm{\epsilon}_t^\ast,\bm{g}_t^\ast)
=
\arg\max_{(\bm{\epsilon}_t,\bm{g}_t)\sim p_{\phi}(\cdot\mid t_1^\ast)}
\widehat V_{\phi}(
\mathcal{T}_{\theta,t:t-1}
(
\bm{S}_t,
(\bm{\epsilon}_t,\bm{g}_t)
\mid \bm{S}_{t_1^\ast}^{\mathrm{ref}}
).
\label{eq:noise_opt}
\end{equation}
and thus,
the optimal next scaffold state as 
\vspace{-8pt}
\begin{equation}
\bm{S}_{t-1}^\ast = \mathcal{T}_{\theta,t:t-1}(\bm{S}_t,(\bm{\epsilon}_t^\ast,\bm{g}_t^\ast)).
\end{equation}
By iteratively selecting the 
optimal
scaffold states along the edited noise trajectories,
\method provides a new transition distribution to address scaffold hopping (Eq.~{\ref{single_hop}}).

\section{Experiments}

\subsection{Experimental Setting}
\label{sec: exp}

\paragraph{Datasets}

We adapt the SBDD models using 
the training sets of the original models, which are all derived from CrossDocked2020 training data
~\citep{francoeur_crossdock}.
We evaluate {\method} on the test set introduced by {\citet{gao_conditardev}}, excluding pocket-ligand 
complexes for which either no or all ligand atoms are identified 
as functional-group atoms, resulting in 60 complexes with 58 unique ligands.
We use this test set because it focuses exclusively on human protein targets related to major diseases, avoiding the substantial proportion (67\%) of non-human targets in the CrossDocked2020 test set
and providing a more human-disease-relevant evaluation setting.
This test set has no overlap with the CrossDocked2020 training set.
The identification of functional-group atoms follows the protocol introduced by~\citet{guan_decompdiff}.
Across the 60 test complexes, 
ligands have an average of 28.7 heavy atoms, including 13.8 scaffold atoms and 14.9 functional-group atoms.

\paragraph{Baselines}

We evaluate three types of baselines on scaffold hopping: 
\textbf{(1)} \conditar\citep{gao_conditardev} and \ipdiff\citep{huang_ipdiff}, adapted to scaffold hopping (Appendix~\ref{app:finetune_SBDD}), 
referred to as \ftconditar and \ftipdiff, respectively. 
\textbf{(2)} \diffhopp~\citep{torge_diffhopp};  and 
\textbf{(3)} \shepherd~\citep{adams_shepherd},
due to their strong generation performance and the availability of public checkpoints for adaptation.
These three models, \ftconditar, \ftipdiff, and \diffhopp, serve two roles in our experiments: 
they are evaluated directly as scaffold hopping baselines, 
and they are also used as the base models on which \method performs inference-time 
editing.
Details on datasets are available in Appendix~\ref{appx:settings:baselines}.

\paragraph{Evaluation Metrics}

We evaluate the generated scaffolds from two perspectives: 
\textbf{(1)} similarities with respect to the reference ligands (\simtwod, \simthreed), and 
\textbf{(2)} general physicochemical 
properties for drug development (Vina S/M/D, QED, SA, Conn.(\%)).
These metrics are described in detail in Appendix~\ref{appx:settings:metrics}.

\paragraph{Iterative Hopping}

To explore more structurally distinct scaffolds, 
we use a generated scaffold $\bm{S}^{\mathrm{gen}}$ 
from \method as a new ``reference" ligand and apply \method over $\bm{S}^{\mathrm{gen}}$
for a second ``hop."
The details on $\bm{S}^{\mathrm{gen}}$ selection are provided in Appendix~\ref{appx:molecule_selection}.
We denote \method with two hops as \methodite, and the original \method with one hop as \methodonehop. 
Note that in \methodite, the reward (Eq.~\ref{eqn:reward}) in the second hop 
is evaluated against the original reference $\bm{S}^{\mathrm{ref}}$.
 
\paragraph{Experimental Protocol}

In \method, we use
the same hyperparameters as the corresponding base models.
Because \method performs atom-level trajectory recovery of the reference scaffold, 
the edited scaffold is constrained to have the same number of atoms as the reference.
To ensure fair comparison,
we use the same scaffold size for all baselines.
A potential pitfall of encouraging greater 2D structural novelty is that the generated molecules may deviate from the favorable chemical characteristics of the reference, leading to less desirable drug-like properties.
Therefore,
we introduce a second setting for \methodonehop and \methodite where \simthreed is replaced with QED in the reward function (Eq.~\ref{eqn:reward}).
Results under this reward function are indicated by ``(QED)".
The implementation details are provided in~\ref{appx:implementation}.

\subsection{Experimental Results}
\label{sec:exp:results}

\subsubsection{Overall Comparison}
\label{sec:exp:results:overall}

\paragraph{Comparison of \method with Baselines}

\begin{table}[htbp]
    \centering
    \caption{
    Comparison of \method with baselines on the test set with $L=100$.
    Best and second-best results are shown in \textbf{bold} and \underline{underlined}, respectively.
                    }
    \label{tab:overall_comparison}

	\begin{small}
    \begin{threeparttable}
    \setlength{\tabcolsep}{4pt}

    \begin{tabular*}{\textwidth}{
        @{\extracolsep{\fill}}
        p{1.35cm}
        l
        r
        r
        r
        r
        r
        r
        r
        r
    }
    \toprule
    &
    \multirow{2}{*}{Model}
    & \multicolumn{2}{c}{Similarity}
    & \multicolumn{3}{c}{Vina}
    & \multirow{2}{*}{QED $\uparrow$}
    & \multirow{2}{*}{SA $\uparrow$}
    & \multirow{2}{*}{\shortstack[c]{Conn.\\(\%) $\uparrow$}} \\
    \cmidrule(lr){3-4}
    \cmidrule(lr){5-7}
    &
    &
    \simtwod $\downarrow$
    & \simthreed $\uparrow$
    & Vina S $\downarrow$
    & Vina M $\downarrow$
    & Vina D $\downarrow$
    &
    &
    \\
    \midrule

    \multirow{4}{*}{\shortstack[c]{Baselines}}
    & \ftconditar
    & $0.398$ & $0.795$
    & $-7.192$ & $-7.753$ & $-8.564$
    & $0.455$ & $0.622$ & $85.2$ \\

    & \ftipdiff
    & $0.420$ & $0.856$
    & $\mathbf{-7.706}$ & $\mathbf{-8.176}$ & $\mathbf{-8.897}$
    & $0.456$ & $0.619$ & $94.7$ \\

    & \diffhopp
    & $0.426$ & $0.812$
    & $-2.223$ & $-6.209$ & $-8.642$
    & $\mathbf{0.535}$ & $0.667$ & $89.5$ \\

    & \shepherd
    & $0.417$ & $0.881$
    & $-6.471$ & $-7.402$ & $-8.390$
    & $0.404$ & $0.576$ & $29.4$ \\

    \midrule

    \multirow{5}{*}{\methodonehop}
    & \ftconditar
    & $0.354$ & $0.884$
    & $-7.508$ & $-8.049$ & $\underline{-8.833}$
    & $0.427$ & $0.601$ & $91.5$ \\

    & \ftconditarqed
    & $0.354$ & $0.859$
    & $\underline{-7.636}$ & $\underline{-8.163}$ & $-8.828$
    & $0.457$ & $0.610$ & $95.8$ \\

    & \ftipdiff
    & $0.363$ & $\mathbf{0.903}$
    & $-7.422$ & $-7.975$ & $-8.688$
    & $0.408$ & $0.584$ & $93.1$ \\

    & \ftipdiffqed
    & $0.356$ & $\underline{0.897}$
    & $-7.368$ & $-7.975$ & $-8.683$
    & $0.437$ & $0.581$ & $93.8$ \\

    & \diffhopp
    & $0.395$ & $0.867$
    & $-4.593$ & $-6.983$ & $-8.820$
    & $\underline{0.511}$ & $\underline{0.696}$ & $\underline{96.4}$ \\

        \cmidrule{2-10}
    
    \multirow{5}{*}{\methodite}
    & \ftconditar
    & $\underline{0.325}$ & $0.893$
    & $-7.405$ & $-7.988$ & $-8.783$
    & $0.426$ & $0.607$ & $93.7$ \\

    & \ftconditarqed
    & $0.334$ & $0.844$
    & $-7.488$ & $-8.064$ & $-8.801$
    & $0.470$ & $0.604$ & $93.0$ \\

    & \ftipdiff
    & $\mathbf{0.322}$ & $0.892$
    & $-7.088$ & $-7.604$ & $-8.477$
    & $0.403$ & $0.587$ & $95.3$ \\

    & \ftipdiffqed
    & $0.333$ & $0.874$
    & $-7.046$ & $-7.648$ & $-8.550$
    & $0.471$ & $0.595$ & $94.4$ \\

    & \diffhopp
    & $0.372$ & $0.863$
    & $-4.759$ & $-7.109$ & $-8.802$
    & $0.498$ & $\mathbf{0.708}$ & $\mathbf{98.2}$ \\

    \bottomrule
    \end{tabular*}
    \end{threeparttable}
    \end{small}
\end{table}

Table~\ref{tab:overall_comparison} presents a comprehensive comparison between  
\method and all the baselines. 
Overall, \methodonehop
consistently achieves lower \simtwod and higher \simthreed than those 
from corresponding base models (as baselines). 
In terms of  
improvement of \method with a base model (e.g., \methodonehop with \ftconditarqed) 
over its corresponding base model (e.g., \ftconditar), 
\methodonehop achieves 
$11.7\%$ decrease in \simtwod  
and $7.3\%$ increase in \simthreed on average across all base models 
(\shepherd is excluded in this calculation). 
Compared to the baselines, the fundamental difference of \method is in its value-guided editing 
over the reference noise trajectories in the base models -- 
the reference noise trajectories help preserve the reference 3D shape during editing, 
whereas value-guided search directs the generation toward low \simtwod. 
The superior performance of \methodonehop demonstrates that such editing is able to effectively 
steer the scaffold generation toward the desired chemical subspace beyond 
what is explored by the baselines. 
Meanwhile, on Vina S, \methodonehop achieves an average  
improvement $39.1\%$ over \ftconditar and \diffhopp, 
while maintaining 
similar Vina M and Vina D scores compared to those of all the base models. 
In terms of \connectivity, \methodonehop also introduces a noticeable 
average improvement $5.0\%$
over the base models. 
Although \methodonehop exhibits slight degradation in QED and SA scores, 
the resulting values remain within acceptable ranges.
Such slight degradation also reflects \methodonehop's exploration of novel chemical structures essential for 
scaffold hopping.

\paragraph{Base model-specific \method performance}

With \simtwod and \simthreed in the optimization objectives (Eqs.~\ref{single_hop} and~\ref{eqn:reward}), 
\methodonehop exhibits different improvement patterns across the two adapted SBDD models, 
\ftconditar and \ftipdiff.
On \ftconditar, \methodonehop improves both \simtwod and \simthreed substantially, reducing \simtwod by $11.1\%$ while increasing \simthreed by $11.2\%$.
In contrast, on \ftipdiff, which starts with
worse
\simtwod and better
\simthreed
as a baseline compared to \ftconditar,
\methodonehop achieves a larger gain in \simtwod ($13.6\%$) but a more moderate improvement in \simthreed ($5.5\%$). 
These results suggest that \methodonehop can adapt to different SBDD base models by adjusting its optimization focus
based on base models's
generation behaviors.
On \diffhopp, 
\methodonehop also consistently improves both \simtwod and \simthreed,
with improvements of $7.3\%$ and $6.8\%$, respectively. 
Note that \diffhopp has a much worse Vina S ($-2.223$) than the other baselines, 
and therefore, the noise trajectory of the reference ligand, which has high binding affinity and thus 
good Vina S, could be out of distribution for \diffhopp.
The fact that \methodonehop substantially improves Vina S to $-4.593$ over \diffhopp suggests that 
its in-distribution perturbation and value-guided editing could partially compensate for the distribution 
mismatch. 
Overall, these improvements across different base models demonstrate that \methodonehop can serve as a plug-and-play inference-time editing framework that generalizes across different molecular diffusion models.

In addition to the three structure-based models, we compare \methodonehop with the ligand-based model \shepherd, which is specifically trained to generate molecules with shapes similar to the reference.
\methodonehop can achieve \simthreed even higher than that of \shepherd,
demonstrating its strong capability of retaining preferred 3D shapes.

\paragraph{Comparison of Different Reward Functions in \method}

We further investigate whether \method remains effective under various reward functions.
Under \methodonehop, with a reward function combining \simtwod and QED of generated molecules, \ftconditarqed and \ftipdiffqed achieve $7.0\%$ and $7.1\%$ higher QED than \ftconditar and \ftipdiff, respectively. 
The corresponding gains increase to $10.3\%$ and $16.9\%$ under \methodite.
This suggests that \method can accommodate different reward functions, 
and could be extended to other molecular properties through appropriate reward design.
Note that although not explicitly optimized, \simthreed after \simtwod and QED optimization 
in \methodonehop still remains higher than that of the corresponding baseline, with improvements of $8.1\%$ and $4.8\%$ for \ftconditarqed and \ftipdiffqed, respectively. 
This suggests that the reference noise trajectory helps retain favorable 3D shape even when \simthreed is not explicitly included in the reward.

\paragraph{Comparison of \methodonehop and \methodite}

Table~\ref{tab:overall_comparison} also shows the performance changes between \methodonehop and \methodite.
Compared with \methodonehop, \methodite generally further reduces \simtwod while maintaining high \simthreed.
For example, \simtwod decreases from $0.354$ to $0.325$ on \ftconditar and from $0.363$ to $0.322$ on \ftipdiff, while \simthreed remains above $0.890$ for both models.
However, this increased 2D diversification is accompanied by weaker binding affinity: when averaged over \ftconditar and \ftipdiff, with and without QED optimization, Vina S changes from $-7.484$ under \methodonehop to $-7.257$ under \methodite.
This might be due to the fact that the second hop uses a generated molecule as its anchor.
This may cause deviations to accumulate across hops in properties not constrained by the similarity profile,
potentially degrading binding affinity.
Overall, these results highlight a trade-off in \method, 
and more generally, in
scaffold hopping,
between progressively improving the similarity profile and preserving binding affinity over multiple hops.

\subsubsection{Comparison of Random Perturbation and Value-Guided Sampling}
\label{sec:exp:results:search}

\begin{table}[htbp]
\centering
\caption{
Comparison of random perturbation and value-guided sampling on $L=T/10$ and $L=T$ (perturbation segment starting 
from $t_1^\ast$ to the end of the diffusion trajectory), using \ftconditar as the base model.
Best and second-best results are shown in \textbf{bold} and \underline{underlined}, respectively.
}
\label{tab:length_comparison}

\begin{small}
\begin{threeparttable}
\setlength{\tabcolsep}{4pt}

\begin{tabular*}{\textwidth}{
    @{\extracolsep{\fill}}
    >{\centering\arraybackslash}p{1.6cm}
    c
    l
    r
    r
    r
    r
    r
    r
    r
    r
}
\toprule
\multirow{2}{*}{}
& \multirow{2}{*}{$L$}
& \multirow{2}{*}{}
& \multicolumn{2}{c}{Similarity}
& \multicolumn{3}{c}{Vina}
& \multirow{2}{*}{QED $\uparrow$}
& \multirow{2}{*}{SA $\uparrow$}
& \multirow{2}{*}{\shortstack[c]{Conn.\\(\%) $\uparrow$}} \\
\cmidrule(lr){4-5}
\cmidrule(lr){6-8}
&
&
&
\simtwod $\downarrow$
& \simthreed $\uparrow$
& Vina S $\downarrow$
& Vina M $\downarrow$
& Vina D $\downarrow$
&
&
\\
\midrule

\multirow{4}{*}{\shortstack[c]{random\\perturbation\\(Eq.~\ref{eq:noise_perturb})}}
& \multirow{2}{*}{$T/10$}
& \methodonehop
& $0.387$
& $\mathbf{0.898}$
& $-7.742$
& $-8.213$
& $-8.840$
& $0.442$
& $0.614$
& $86.6$ \\

&
& \methodite
&
$0.343$
& $0.874$
& $-7.581$
& $-8.091$
& $-8.813$
& $0.425$
& $0.604$
& $90.0$ \\

\cmidrule{3-11}

&
\multirow{2}{*}{$T$}
& \methodonehop
& $0.391$
& $0.866$
& $\mathbf{-7.837}$
& $\mathbf{-8.308}$
& $\underline{-9.028}$
& $\mathbf{0.462}$
& $\underline{0.629}$
& $92.3$ \\

&
& \methodite
&
$0.356$
& $0.854$
& $-7.763$
& $-8.232$
& $-8.915$
& $0.450$
& $0.620$
& $95.2$ \\

\cmidrule{1-11}

\multirow{4}{*}{\shortstack[c]{value-guided\\sampling\\(Eq.~\ref{eq:noise_opt})}}
& \multirow{2}{*}{$T/10$}
& \methodonehop
& $0.354$
& $0.884$
& $-7.508$
& $-8.049$
& $-8.833$
& $0.427$
& $0.601$
& $91.5$ \\

&
& \methodite
&
$\mathbf{0.325}$
& $\underline{0.893}$
& $-7.405$
& $-7.988$
& $-8.783$
& $0.426$
& $0.607$
& $93.7$ \\

\cmidrule{3-11}

&
\multirow{2}{*}{$T$}
& \methodonehop
& $0.340$
& $0.882$
& $-7.612$
& $-8.157$
& $-8.904$
& $0.440$
& $0.628$
& $95.1$ \\

&
& \methodite
&
$\underline{0.334}$
& $0.887$
& $-7.539$
& $-8.054$
& $-8.806$
& $0.432$
& $0.621$
& $93.8$ \\

\bottomrule
\end{tabular*}
\end{threeparttable}
\end{small}
\end{table}

Table~\ref{tab:length_comparison} compares results with random perturbation (Eq.~\ref{eq:noise_perturb}) and value-guided sampling (Eq.~\ref{eq:noise_opt}) applied in \method.
We choose \ftconditar for this study because \method applied to \ftconditar achieves strong performance and maintaining stable Vina S across hops.
Overall, pairwise comparisons across the 4 settings ($T/10$ vs $T$, \methodonehop vs 
\methodite) show that value-guided sampling reduces average \simtwod by $8.3\%$ and improves \simthreed by $1.6\%$ relative to random perturbation.
This suggests that the performance gains arise from guided search
toward desired properties rather than from perturbation-induced changes alone.

Table~\ref{tab:length_comparison} also compares the performance of \method under two perturbation segment lengths: a shorter segment covering one-tenth of the diffusion trajectory ($L=T/10$) and a longer segment extending from $t_1^\ast$ to the end of the diffusion trajectory (denoted as $L=T$).
We use $T/10$ 
to provide a moderate discretization granularity over the diffusion trajectory,
balancing exploration of distinct molecules and preservation of reference shapes.
With $L=T/10$, when averaged over \methodonehop and \methodite, value-guided sampling improves \simtwod by $6.9\%$ over random perturbation, with a minor improvement in \simthreed of $0.31\%$.
This trend becomes more significant with $L=T$, where value-guided sampling reduces \simtwod by $9.6\%$ and increases \simthreed by $2.9\%$.
The larger gains with $L=T$ may be explained by the longer perturbation segment, which preserves less information from the reference structure and provides greater room 
for value-guided sampling to adjust the similarity profile.

\subsubsection{Comparison of Reference-Optimal $t_1^\ast$ and Fixed $t_1$}
\label{sec:exp:results:t1}

\begin{figure}[h!]
\centering
\begin{minipage}[c]{0.30\textwidth}
    \centering
    \includegraphics[width=.82\linewidth]{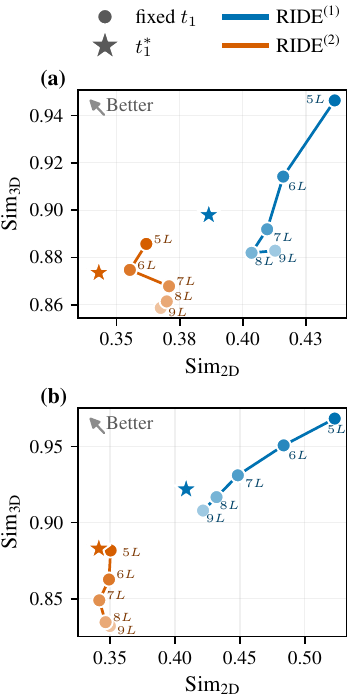}
    \caption{
    Comparison of Sim$_{\text{2D}}$ vs. Sim$_{\text{3D}}$ across $t_1 \in \{nL\}_{n=5}^{9}$ and  $t_1^\ast$ with $L=T/10$. 
	\textbf{(a)} \ftconditar; \textbf{(b)} \ftipdiff. 
    }
    \label{fig:setting_a}
\end{minipage}
\hfill
\begin{minipage}[c]{0.68\textwidth}
    \centering

    \includegraphics[width=.95\linewidth]{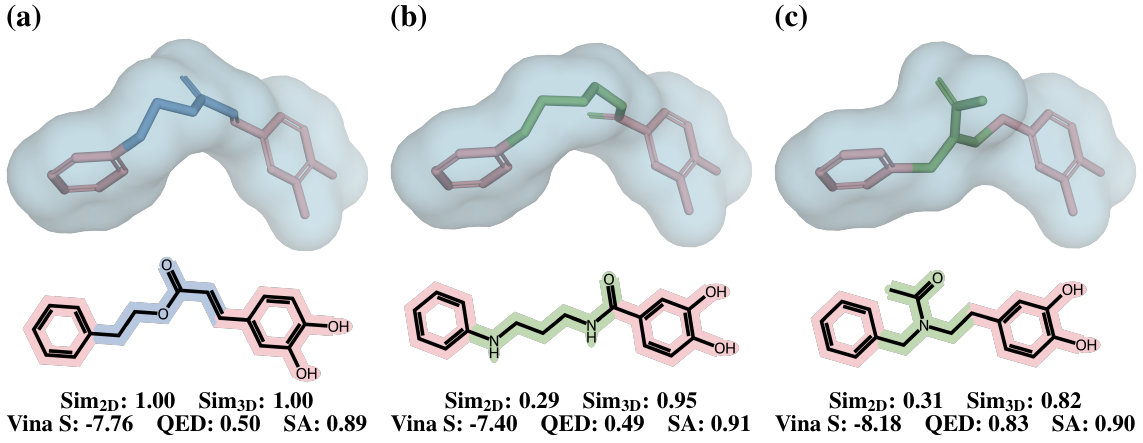}
    \captionof{figure}{
    Example of generated scaffolds for AK1BA (PDB ID: 5liu).    }
    \label{fig:case_5liu}
    \vspace{0.8em}
    \includegraphics[width=\linewidth]{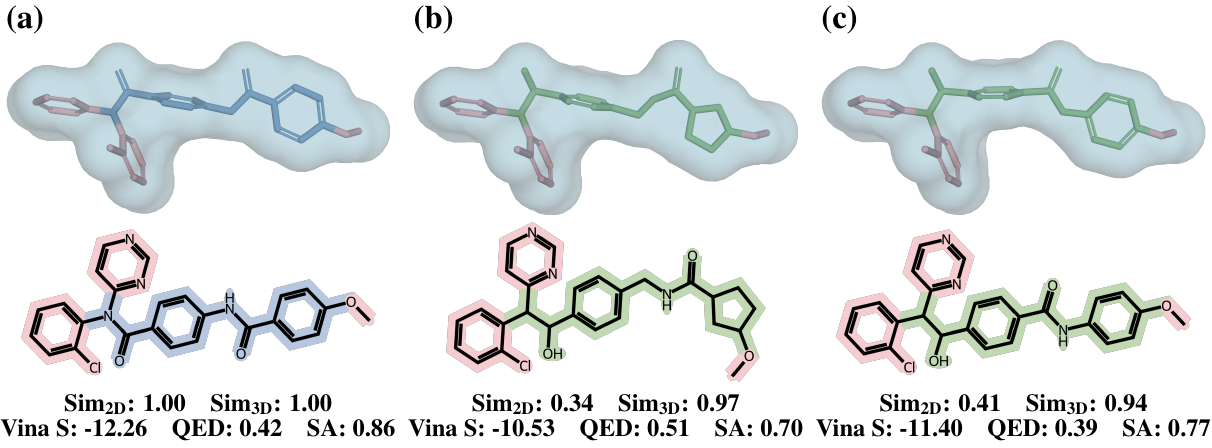}
    \captionof{figure}{
    Example of generated scaffolds for TNKS2 (PDB ID: 5aeh).    }
    \label{fig:case_5aeh}

\end{minipage}
\end{figure}

Figures~\ref{fig:setting_a}a and~\ref{fig:setting_a}b illustrate \simtwod and \simthreed across candidate $t_1$ values and at the reference-optimal $t_1^\ast$, using \ftconditar and \ftipdiff as the base model, respectively, under random perturbation (Eq.~\ref{eq:noise_perturb}).
On both hops,
compared with a fixed $t_1$,
the reference-optimal $t_1^\ast$ achieves the best \simtwod while with competitive \simthreed,
indicating that the reference-optimal $t_1^\ast$ identifies a strong starting point for subsequent value-guided sampling.
Meanwhile, \methodonehop and \methodite
exhibit different trends in \simtwod and \simthreed across candidate $t_1$ values.
With \methodonehop, a later perturbation improves \simthreed at the cost of degrading \simtwod, 
which is consistent with the intuition that later perturbation preserves more of the reference shapes.
A similar pattern in \simtwod and \simthreed
is observed with \methodite, but with lower \simtwod and \simthreed than those in 
\methodonehop and smaller variations across candidate $t_1$ values,
indicating that the \methodite
is less sensitive to the choice of perturbation segment.
This may be due to the fact that
the \methodite
is anchored to a generated molecule that is already structurally distant from the original reference. 
Therefore, further perturbations can induce only limited additional changes in the similarity profile with respect to the reference molecule.

\vspace{-4pt}
\subsubsection{Case Studies}
\label{sec:exp:results:case}
\vspace{-3pt}
 
Figures~\ref{fig:case_5liu}b and~\ref{fig:case_5liu}c show two generated scaffolds by \method 
for AK1BA, a protein target involved in gastrointestinal cancers~\citep{wang2026pathway}.
Both generated scaffolds have low \simtwod and high \simthreed with the reference scaffold, as in Figure~\ref{fig:case_5liu}a.
However, these scaffolds differ greatly in flexibility.
The scaffold in Figure~\ref{fig:case_5liu}b is more flexible, with 6 rotatable bonds in the linker 
compared to 5 in that of the reference scaffold.
Flexible molecules can better adapt to protein conformation shifts or resistance mutations in cancer~\citep{fang2014conformational}. 
However, such flexibility often incurs an entropic penalty to the binding energies; this aligns with the slightly lower Vina S observed here. In contrast, the scaffold in Figure~\ref{fig:case_5liu}c has
a shorter, more rigid linker
than the reference,
with only 4 rotatable linker bonds.
More rigid molecules
tend to have lower entropic binding penalties and higher selectivity, 
improving effectiveness while reducing
off-target effects~\citep{fang2014conformational},
Together, these examples show \method can generate scaffolds with different 
structures and properties.

Figures~\ref{fig:case_5aeh} presents another example from \method for 
TNKS2,
a target implicated in many cancers~\citep{huang2009tankyrase}.
Both generated scaffolds (Figures~\ref{fig:case_5aeh}b-c)
have moderately low \simtwod and very high \simthreed with the reference scaffold in 
Figure~\ref{fig:case_5aeh}a.
The generated scaffolds share some
common deviations from the reference: both use carbon to bind the two cyclic functional groups
and replace one ketone with an alcohol.
These changes alter physicochemical properties (e.g., pKa, lipophilicity) that influence absorption, distribution, metabolism, excretion, and toxicity properties of molecules.
The scaffold in Figure~\ref{fig:case_5aeh}b also uses cyclopentane in place of 
one benzene and adds
an additional rotatable bond.
These changes generally increase flexibility. 
Meanwhile, the generated scaffolds retain some key reference scaffold features -- 
a benzene ring connected to another ring by a linker with a ketone.
This helps the generated scaffolds attain very high \simthreed,
while still introducing substantial 2D variation via atom-type, bond-type, ring-size, and linker changes.

\section{Conclusion}

\method is a novel inference-time diffusion editing framework for scaffold hopping.
Our experimental results demonstrate that \method is effective, 
consistently generating scaffolds with low 2D
similarity and high 3D similarity to the reference scaffold.
Furthermore, \method has demonstrated its effectiveness under a variety of different settings,
including different base models, reward functions, perturbation segment lengths, and number of scaffold hopping iterations.
These results motivate future work to extend
the framework to optimize additional molecular properties and further improve generated molecule quality.
Future work may also investigate the development of more adaptive perturbation strategies for even more effective optimization.


\section*{Acknowledgments}

This project was made possible, in part, by support from the National Science Foundation
grant no. IIS-2435819 (X.N.), the National Library of Medicine grant no. 1R01LM014385 (X.N.), and the Sanofi iDEA-TECH Awards North America (X.N.).
Any opinions, findings, conclusions or recommendations expressed in this
paper are those of the authors and do not necessarily reflect the views of the funding agency.
The authors would like to thank Benjamin Burns and Ye Liu for their proofreading of the manuscript.
The authors would also like to thank Daniel Adu-ampratwum for a helpful discussion regarding case study molecules.

\section*{AI Use Statement}
During the preparation of this work, the authors used AI tools to assist with language editing, improving the clarity of the manuscript, and identifying relevant literature.
All AI-assisted content was reviewed and verified by the authors.

\section*{Ethics Statement}

\method is a novel inference-time diffusion editing framework for scaffold hopping, which is an important problem in chemistry.
While we have not designed \method for any unsafe  or harmful purpose,
we acknowledge that not
all generated scaffolds and molecules are safe,
and that \method could generate harmful content.
Therefore, we
strongly encourage responsible human expert supervision for any use of \method or its generated scaffolds and molecules.
Specifically, human expert chemists should verify the safety of any generated scaffolds and molecules before any attempt is made to synthesize or test them in the laboratory.
Furthermore, we exhort
all users of \method to follow all applicable
safety regulations, ethical guidelines, laws, and professional best practices.

\section*{Reproducibility statement}
All the source code has been made available at \url{https://anonymous.4open.science/r/RIDE-C8A0}. All the datasets are public. 
We reported all the experimental protocal in Section~\ref{sec: exp},  
and all the hyperparameters in Appendix~\ref{appx:molecule_selection}.

\bibliography{iclr2027_conference}
\bibliographystyle{iclr2027_conference}

\newpage
\appendix
\section{Appendix}

\subsection{Adapting Diffusion-based SBDD Models to Scaffold Hopping}
\label{app:finetune_SBDD}

To adapt a pretrained diffusion-based SBDD model, 
denoted as $p_\theta((\bm{S},\bm{F})\mid \bm{P})$, 
to a base model $p_\theta(\bm{S}\mid \bm{P},\bm{F}^{\mathrm{ref}})$ for scaffold hopping, 
with $\bm{F}^{\mathrm{ref}}$ as additional conditioning, 
we fix $\bm{F}^{\mathrm{ref}}$ and restrict the diffusion loss to 
scaffold atoms and update the diffusion parameters as follows,
\begin{equation}
\min\nolimits_{{\theta}}\mathbb{E}_{t,\bm{M}_t}
\left[
\sum\nolimits_{i=1}^{n}
\mathbb{I}(\bm{a}_i\in \bm{S}_t)
\mathcal{L}_{\theta}^{i}
(\bm{M}_t;\bm{P},t)
\right],
\label{eq}
\end{equation}
where $\bm{M}_t=(\bm{S}_t,\bm{F}^{\mathrm{ref}})$ denotes the partially noised ligand with fixed functional groups at $t$,
$\mathbb{I}(\bm{a}_i\in \bm{S}_t)$ 
indicates whether the $i$-th atom belongs to the 
noised scaffold $\bm{S}_t$,
and $\mathcal{L}_{\mathrm{\theta}}^{i}$ denotes the SBDD diffusion loss on the $i$-th atom.

Although the additional conditioning can also
be imposed through inpainting 
during inference time of SBDD models,
inpainting can create a mismatch with the distribution $p_\theta((\bm{S},\bm{F})\mid \bm{P})$
learned during training,
potentially leading to invalid molecular structures and unfavorable conformations.
The distribution mismatch may also make the model more vulnerable to modifications
on the diffusion trajectory,
thereby reducing 
 control over the generated structures.
Compared to training a conditional generative model from scratch or inpainting,
adapting a pretrained SBDD model offers multiple benefits: 
Adaptation leverages the pretrained SBDD model's capacity as a strong prior for generating chemically valid molecules with favorable binding affinity and 3D shapes;  
it also optimizes the model parameters to accommodate the additional conditioning ($\bm{F}^\mathrm{ref}$).

\subsection{Algorithms}
\label{appx:algorithm}

\subsubsection{\method}

\begin{algorithm}[hthp]
\caption{\method 
}
\label{alg:method}
\begin{algorithmic}[1]

\Require Protein pocket $\bm{P}$, reference scaffold and functional groups $(\bm{S}^{\mathrm{ref}},\bm{F}^{\mathrm{ref}})$, perturbation length $L$, diffusion model $\theta$
\Ensure Generated scaffold $\bm{S}_0$

\Statex \textbf{Reference noise trajectory recovery}

\State $\{(\bm{X}_t^{\mathrm{ref}}, \bm{V}_t^{\mathrm{ref}})\}_{t=0}^{T}$,
$\{(\bm{\epsilon}_t^{\mathrm{ref}},\bm{g}_t^{\mathrm{ref}})\}_{t=1}^{T}
=$ \Call{Inversion}{$\bm{S}^{\mathrm{ref}}, \bm{P}, \bm{F}^{\mathrm{ref}}$}
\Comment{Details in Algorithm~\ref{alg:scaffold_inversion}}

\Statex
\vspace{-0.7em}
\Statex \textbf{Reference-optimal perturbation segment selection}

\State $t_1^\ast =$ \Call{SegmentSelection}{
$\{(\bm{X}_t^{\mathrm{ref}},\bm{V}_t^{\mathrm{ref}})\}$,
$\{(\bm{\epsilon}_t^{\mathrm{ref}},\bm{g}_t^{\mathrm{ref}})\},
\bm{P}, \bm{F}^{\mathrm{ref}}, L$}
\Comment{Details in Algorithm~\ref{alg:t1_selection}}

\State $t_2^\ast \gets \max(t_1^\ast-L, 0)$

\Statex
\vspace{-0.7em}
\Statex \textbf{Value-guided scaffold sampling}

\State $\bm{S}_{t_2^\ast} =$ \Call{ValueGuidedSampling}{
$\{(\bm{X}_t^{\mathrm{ref}},\bm{V}_t^{\mathrm{ref}})\}$,
$\{(\bm{\epsilon}_t^{\mathrm{ref}},\bm{g}_t^{\mathrm{ref}})\},
\bm{P}, \bm{F}^{\mathrm{ref}},
t_1^\ast,
t_2^\ast$}
\Comment{Details in Algorithm~\ref{alg:value}}

\State $\bm{S}_0
= \mathcal{T}_{\theta,t_2^\ast:0}^{\mathrm{ref}}
(\bm{S}_{t_2^\ast})$
\Comment{Reuse reference noise}

\State \Return $\bm{S}_0$

\end{algorithmic}
\end{algorithm}
\method is summarized in Algorithm~\ref{alg:method}. 
The following sections describe each component in detail.

\subsubsection{Reference Trajectory Inversion}
The inversion of reference scaffold is summarized in Algorithm~\ref{alg:scaffold_inversion}.

\begin{algorithm}[htbp]
\caption{Trajectory Inversion of $\bm{S}^\mathrm{ref}$}
\label{alg:scaffold_inversion}
\begin{algorithmic}[1]

\Function{Inversion}{$\bm{S}^{\mathrm{ref}}, \bm{P}, \bm{F}^{\mathrm{ref}}$}

\State $(\bm{X}_0^{\mathrm{ref}},\bm{V}_0^{\mathrm{ref}})
\gets (\bm{X}^{\mathrm{ref}},\bm{V}^{\mathrm{ref}})$

\For{$t=1,\ldots,T$}

    \State $\hat{\bm{\epsilon}}
    \gets
    \bm{\epsilon}_\theta
    (
    \bm{X}_{t-1}^{\mathrm{ref}},
    t-1;\bm{V}_{t-1}^{\mathrm{ref}},\bm{P},\bm{F}^{\mathrm{ref}})$

    \State $\displaystyle
    \bm{X}_t^{\mathrm{ref}}
    \gets
    \sqrt{\bar{\alpha}_t}
    (
    \frac{
    \bm{X}_{t-1}^{\mathrm{ref}}
    -
    \sqrt{1-\bar{\alpha}_{t-1}}\,
    \hat{\bm{\epsilon}}
    }{
    \sqrt{\bar{\alpha}_{t-1}}
    }
    )
    +
    \sqrt{1-\bar{\alpha}_t}\,
    \hat{\bm{\epsilon}}$

    \State $\displaystyle
    \bm{V}_{t}^{\mathrm{ref}}
    \sim
    q\!\left(
    \bm{V}_{t}
    \mid
    \bm{V}_{t-1}^{\mathrm{ref}}
    \right)$

\EndFor

\For{$t=1,\ldots,T$}

    \State $\displaystyle
    \bm{\epsilon}_t^{\mathrm{ref}}
    \gets
    \frac{
    \bm{X}_{t-1}^{\mathrm{ref}}
    -
    \bm{\mu}_\theta
    (
    \bm{X}_t^{\mathrm{ref}},
    t;\bm{V}_t^{\mathrm{ref}},
    \bm{P},\bm{F}^{\mathrm{ref}}
    )
    }{\sigma_t}$

    \State $\hat{{\bm{p}}}_{t}
        \gets
        p_\theta(
        {\bm{V}}_{t-1}
        \mid
        \bm{X}_t^{\mathrm{ref}},
        \bm{V}_t^{\mathrm{ref}};
        \bm{P},
        \bm{F}^{\mathrm{ref}}
        )$

    \State $\displaystyle
    c_{i}^{\mathrm{ref}}
    \gets
    {V}_{t-1,i}^{\mathrm{ref}},
    \quad i=1,\ldots,n$
    
    \State $\displaystyle g_{t,i,k}^{\mathrm{ref}} \gets \begin{cases} \left[ \max\limits_{j\neq c_i^{\mathrm{ref}}} \log \hat{p}_{t,i,j} - \log \hat{p}_{t,i,c_i^{\mathrm{ref}}} + \delta \right]_+, & k=c_i^{\mathrm{ref}}, \\[2pt] 0, & k\neq c_i^{\mathrm{ref}}, \end{cases} $ \Statex $\hfill i=1,\ldots,n,\quad k=1,\ldots,d$

\EndFor

\State \Return
$\{(\bm{X}_t^{\mathrm{ref}},\bm{V}_t^{\mathrm{ref}})\}_{t=0}^{T}$,
$\{(\bm{\epsilon}_t^{\mathrm{ref}},\bm{g}_t^{\mathrm{ref}})\}_{t=1}^{T}$
\EndFunction
\end{algorithmic}
\end{algorithm}

\subsubsection{Optimal $t_1$ Selection}
The procedure for selecting $t_1^\ast$ is summarized in Algorithm~\ref{alg:t1_selection}.

\begin{algorithm}[htbp]
\caption{Selection of $t_1$}
\label{alg:t1_selection}
\begin{algorithmic}[1]

\Function{SegmentSelection}{$\{(\bm{X}_t^{\mathrm{ref}},\bm{V}_t^{\mathrm{ref}})\}$,
$\{(\bm{\epsilon}_t^{\mathrm{ref}},\bm{g}_t^{\mathrm{ref}})\},\bm{P}, \bm{F}^{\mathrm{ref}},L$}

\For{$t_1 \in \{\tau_n=\frac{nT}{N}: n=n_1,\ldots,n_2\}$}
    \For{$k=1,\ldots,K$}
        \State
        $\{(\bm{\epsilon}_t^{(k)},\bm{g}_t^{(k)})\}
        \sim
        p_{\phi}(\cdot\mid t_1)$

        \State
        $\bm{S}^{(k)}(t_1)
        \gets
        \mathcal{T}_{\theta,t_1:0}
        (
        \bm{S}_{t_1}^{\mathrm{ref}},
        \{(\bm{\epsilon}_t^{(k)},\bm{g}_t^{(k)})\}
        )$
    \EndFor

    \State
    $J(t_1\mid \bm{P}, \bm{F}^{\mathrm{ref}}, \bm{S}^{\mathrm{ref}}, L)
    \gets
    \frac{1}{K}\sum_{k=1}^{K}
    \Big[
    \eta(1-\operatorname{Sim}_{2\mathrm{D}}(\bm{S}^{(k)}(t_1),\bm{S}^{\mathrm{ref}}))
    +(1-\eta)\operatorname{Sim}_{3\mathrm{D}}(\bm{S}^{(k)}(t_1),\bm{S}^{\mathrm{ref}})
    \Big]$
\EndFor

\State
$t_1^\ast
\gets
\arg\max_{t_1}
J(t_1\mid \bm{P}, \bm{F}^{\mathrm{ref}}, \bm{S}^{\mathrm{ref}}, L)$

\State \Return $t_1^\ast$

\EndFunction
\end{algorithmic}
\end{algorithm}

\subsubsection{Value-guided sampling}
The procedure for value estimation and value-guided sampling is summarized in Algorithm~\ref{alg:value}.

\begin{algorithm}[htbp]
\caption{Value-Guided Sampling}
\label{alg:value}
\begin{algorithmic}[1]

\Function{ValueGuidedSampling}{$\{(\bm{X}_t^{\mathrm{ref}},\bm{V}_t^{\mathrm{ref}})\}$,
$\{(\bm{\epsilon}_t^{\mathrm{ref}},\bm{g}_t^{\mathrm{ref}})\},\bm{P}, \bm{F}^{\mathrm{ref}},
t_1^\ast,
t_2^\ast$}

\State $\bm{S}_{t_1^\ast}\gets \bm{S}_{t_1^\ast}^{\mathrm{ref}}$

\For{$m=1,\ldots,M$}
    \State $\bm{S}_{t_2^\ast}^{(m)}
    \sim p_\theta(
    \bm{S}_{t_2^\ast}
    \mid
    \bm{S}_{t_1^\ast}^{\mathrm{ref}}
    )$, \quad
    $R^{(m)}
    =
    \mathcal{R}(
    \mathcal{T}_{\theta,t_2^\ast:0}^{\mathrm{ref}}
    (\bm{S}_{t_2^\ast}^{(m)}),
    \bm{S}^{\mathrm{ref}}
    )$
    \Comment{Prepare lookahead samples}
\EndFor

\For{$t=t_1^\ast,\ldots,t_2^\ast+1$}
    \If{$t$ is a guided timestep}

        \For{$b=1,\ldots,B$}
            \State $(\bm{\epsilon}_t^{(b)},\bm{g}_t^{(b)})
            \sim p_{\phi}(\cdot\mid t_1^\ast)$, \quad
            $\bm{S}_{t-1}^{(b)}
            =
            \mathcal{T}_{\theta,t:t-1}(
            \bm{S}_t,
            (\bm{\epsilon}_t^{(b)},\bm{g}_t^{(b)})
            )$
            \Comment{Propose a candidate for selection}

            \State $\widehat{V}^{(b)}
            =
            \widehat{V}_\phi(
            \bm{S}_{t-1}^{(b)}
            \mid
            \bm{S}_{t_1^\ast}^{\mathrm{ref}}
            )$
            \Comment{Estimate value using Eq.~\ref{eq:value_est_mc}}
        \EndFor

        \State $b^\ast=\arg\max_b\widehat V^{(b)}$, \quad
        $\bm{S}_{t-1}\gets\bm{S}_{t-1}^{(b^\ast)}$
        \Comment{Value-guided selection}

    \Else

        \State $(\bm{\epsilon}_t,\bm{g}_t)
        \sim p_{\phi}(\cdot\mid t_1^\ast)$, \quad
        $\bm{S}_{t-1}
        =
        \mathcal{T}_{\theta,t:t-1}(
        \bm{S}_t,
        (\bm{\epsilon}_t,\bm{g}_t)
        )$
        \Comment{Reverse transition without selection}

    \EndIf
\EndFor

\State \Return $\bm{S}_{t_2^\ast}$

\EndFunction

\end{algorithmic}
\end{algorithm}

\subsection{Additional theories}
\subsubsection{Proof for Theorem 1}
\label{appx:theorem1}
\begin{proof}
By definition, the value of an intermediate state $\bm{S}_t$ is the expected terminal reward obtained by continuing the reverse process from $\bm{S}_t$.

Then,
\begin{equation*}
\begin{aligned}
V_\phi(
\bm{S}_t
\mid
\bm{S}_{t_1^\ast}^{\mathrm{ref}}
)
&= \mathbb{E}_{
p_{\theta,\phi}(
\bm{S}_{0}
\mid
\bm{S}_t,\bm{S}_{t_1^\ast}^{\mathrm{ref}}
)
}
\left[
\mathcal{R}(\bm{S}_{0},
\bm{S}^{\mathrm{ref}}
)
\right] \\
&= \int
\int
\mathcal{R}(
\bm{S}_0,
\bm{S}^{\mathrm{ref}}
)
p_{\theta,\phi}(
\bm{S}_0\mid \bm{S}_{t_2^\ast},\bm{S}_{t_1^\ast}^{\mathrm{ref}}
)
p_{\theta,\phi}(
\bm{S}_{t_2^\ast}\mid \bm{S}_t,\bm{S}_{t_1^\ast}^{\mathrm{ref}}
)
d\bm{S}_0d\bm{S}_{t_2^\ast} \\
&=
\int
\int
\mathcal{R}(
\bm{S}_0,
\bm{S}^{\mathrm{ref}}
)
\delta(
\bm{S}_0-
\mathcal{T}_{\theta,t_2^\ast:0}^{\mathrm{ref}}
(\bm{S}_{t_2^\ast})
)
p_{\theta,\phi}(
\bm{S}_{t_2^\ast}\mid \bm{S}_t,\bm{S}_{t_1^\ast}^{\mathrm{ref}}
)
d\bm{S}_0d\bm{S}_{t_2^\ast} \\
&= 
\int
\mathcal{R}(
\mathcal{T}_{\theta,t_2^\ast:0}^{\mathrm{ref}}
(\bm{S}_{t_2^\ast}),
\bm{S}^{\mathrm{ref}}
)
p_{\theta,\phi}(
\bm{S}_{t_2^\ast}\mid \bm{S}_t,\bm{S}_{t_1^\ast}^{\mathrm{ref}}
)
d\bm{S}_{t_2^\ast} \\
&= 
\int \mathcal{R}(\mathcal{T}_{\theta, t_2^\ast:0}^{\mathrm{ref}}(\bm{S}_{t_2^\ast}), \bm{S}^{\mathrm{ref}}) \frac{p_{\theta,\phi}(\bm{S}_{t_2^\ast}, \bm{S}_t \mid \bm{S}_{t_1^\ast}^{\mathrm{ref}})}{p_{\theta,\phi}(\bm{S}_t \mid \bm{S}_{t_1^\ast}^{\mathrm{ref}})} d \bm{S}_{t_2^\ast} \\
&=
\int \mathcal{R}(\mathcal{T}_{\theta, t_2^\ast:0}^{\mathrm{ref}}(\bm{S}_{t_2^\ast}), \bm{S}^{\mathrm{ref}}) \frac{p_{\theta,\phi}(\bm{S}_t \mid \bm{S}_{t_2^\ast}, \bm{S}_{t_1^\ast}^{\mathrm{ref}}) p_{\theta,\phi}(\bm{S}_{t_2^\ast} \mid \bm{S}_{t_1^\ast}^{\mathrm{ref}})}{p_{\theta,\phi}(\bm{S}_t \mid \bm{S}_{t_1^\ast}^{\mathrm{ref}})} d S_{t_2^\ast}  
\\
&=
\frac{\int \mathcal{R}(\mathcal{T}_{\theta, t_2^\ast:0}^{\mathrm{ref}}(\bm{S}_{t_2^\ast}), \bm{S}^{\mathrm{ref}}) p_{\theta,\phi}(\bm{S}_t \mid \bm{S}_{t_2^\ast}, \bm{S}_{t_1^\ast}^{\mathrm{ref}}) p_{\theta,\phi}(\bm{S}_{t_2^\ast} \mid \bm{S}_{t_1^\ast}^{\mathrm{ref}} ) d \bm{S}_{t_2^\ast}}{\int p_{\theta,\phi}(\bm{S}_t \mid \bm{S}_{t_2^\ast}, \bm{S}_{t_1^\ast}^{\mathrm{ref}}) p_{\theta,\phi}(\bm{S}_{t_2^\ast} \mid \bm{S}_{t_1^\ast}^{\mathrm{ref}} ) d \bm{S}_{t_2^\ast}} \\
&=
\frac{\int \mathcal{R}(\mathcal{T}_{\theta, t_2^\ast:0}^{\mathrm{ref}}(\bm{S}_{t_2^\ast}), \bm{S}^{\mathrm{ref}}) q(\bm{S}_t \mid \bm{S}_{t_2^\ast}, \bm{S}_{t_1^\ast}^{\mathrm{ref}}) p_{\theta,\phi}(\bm{S}_{t_2^\ast} \mid \bm{S}_{t_1^\ast}^{\mathrm{ref}} ) d \bm{S}_{t_2^\ast}}{\int q(\bm{S}_t \mid \bm{S}_{t_2^\ast}, \bm{S}_{t_1^\ast}^{\mathrm{ref}}) p_{\theta,\phi}(\bm{S}_{t_2^\ast} \mid \bm{S}_{t_1^\ast}^{\mathrm{ref}} ) d \bm{S}_{t_2^\ast}} \\
&=
\frac{\int \mathcal{R}(\mathcal{T}_{\theta, t_2^\ast:0}^{\mathrm{ref}}(\bm{S}_{t_2^\ast}), \bm{S}^{\mathrm{ref}}) \cancel{q(\bm{S}_{t_1^\ast}^{\mathrm{ref}} \mid \bm{S}_t)} \frac{q(\bm{S}_t \mid \bm{S}_{t_2^\ast})}{q(\bm{S}_{t_1^\ast}^{\mathrm{ref}} \mid \bm{S}_{t_2^\ast})} p_{\theta,\phi}(\bm{S}_{t_2^\ast} \mid \bm{S}_{t_1^\ast}^{\mathrm{ref}} ) d \bm{S}_{t_2^\ast}}{\int \cancel{q(\bm{S}_{t_1^\ast}^{\mathrm{ref}} \mid \bm{S}_t)}\frac{q(\bm{S}_t \mid \bm{S}_{t_2^\ast})}{q(\bm{S}_{t_1^\ast}^{\mathrm{ref}} \mid \bm{S}_{t_2^\ast})} p_{\theta,\phi}(\bm{S}_{t_2^\ast} \mid \bm{S}_{t_1^\ast}^{\mathrm{ref}} ) d \bm{S}_{t_2^\ast}} \\
&=
\frac{
\mathbb{E}_{
p_{\theta,\phi}(
\bm{S}_{t_2^\ast}
\mid
\bm{S}_{t_1^\ast}^{\mathrm{ref}}
)
}
\left[
\frac{
q(\bm{S}_t\mid \bm{S}_{t_2^\ast})
}{
q(\bm{S}_{t_1^\ast}^{\mathrm{ref}}\mid \bm{S}_{t_2^\ast})
}
\mathcal{R}(
\mathcal{T}_{\theta,t_2^\ast:0}^{\mathrm{ref}}(\bm{S}_{t_2^\ast}),
\bm{S}^{\mathrm{ref}}
)
\right]
}{
\mathbb{E}_{
p_{\theta,\phi}(
\bm{S}_{t_2^\ast}
\mid
\bm{S}_{t_1^\ast}^{\mathrm{ref}}
)
}
\left[
\frac{
q(\bm{S}_t\mid \bm{S}_{t_2^\ast})
}{
q(\bm{S}_{t_1^\ast}^{\mathrm{ref}}\mid \bm{S}_{t_2^\ast})
}
\right]
} 
\label{eq:proof-value-definition}
\end{aligned}
\end{equation*}
\end{proof}

\subsubsection{Theorem 2}
\label{appx:theorem2-0}

\begin{theorem}
Given $M$ samples
$\{\bm{S}_{t_2^\ast}^{(m)}\}_{m=1}^{M}$
from $p_{\theta,\phi}(\bm{S}_{t_2^\ast}\mid \bm{S}_{t_1^\ast}^{\mathrm{ref}})$,
the value of an intermediate state $\bm{S}_t$ in Eq.~\ref{eq:value_estimator} 
can be estimated as
\vspace{-5pt}
    \begin{equation}
    \widehat V_{\phi}(
    \bm{S}_t\mid \bm{S}_{t_1^\ast}^{\mathrm{ref}}
    )
    =
    \sum_{m=1}^{M}
    \operatorname{softmax}_m(\{\ell_t^{(j)}\}_{j=1}^M
    )
    \mathcal R(
    \mathcal T_{\theta,t_2^\ast:0}^{\mathrm{ref}}
    (\bm{S}_{t_2^\ast}^{(m)}),
    \bm{S}^{\mathrm{ref}}
    ),
    \label{eq:value_est_mc}
    \end{equation}
where
\vspace{-10pt}
    \begin{equation}
    \ell_t^{(m)}
    =
    \log q(
    \bm{S}_t\mid \bm{S}_{t_2^\ast}^{(m)}
    )
    -
    \log q(
    \bm{S}_{t_1^\ast}^{\mathrm{ref}}
    \mid
    \bm{S}_{t_2^\ast}^{(m)}
    ).
    \label{eq:logtransition}
    \end{equation}
\end{theorem}
The two log-transition terms in $\ell_t^{(m)}$ are calculated in closed form using the Gaussian forward kernel for atom positions and the categorical forward kernel for atom types.
We provide the proof and detailed derivation of the log-transition terms below.

\paragraph{Proof for theorem 2}
\label{appx:theorem2}
\begin{proof}
By applying a Monte Carlo approximation to the expectations in Eq.~\ref{eq:value_estimator}, we obtain
\begin{equation*}
\widehat{V}_\phi(\bm{S}_t \mid \bm{S}_{t_1^\ast}^{\mathrm{ref}} )
=
\frac{\sum_{m=1}^M \frac{q(\bm{S}_t \mid \bm{S}_{t_2^\ast}^{(m)})}{q(\bm{S}_{t_1^\ast}^{\mathrm{ref}} \mid \bm{S}_{t_2^\ast}^{(m)})} \mathcal{R}(\mathcal{T}_{\theta, t_2^\ast:0}^{\mathrm{ref}}(\bm{S}_{t_2^\ast}^{(m)}), \bm{S}^{\mathrm{ref}})}{\sum_{m=1}^M \frac{q(\bm{S}_t \mid \bm{S}_{t_2^\ast}^{(m)})}{q(\bm{S}_{t_1^\ast}^{\mathrm{ref}} \mid \bm{S}_{t_2^\ast}^{(m)})}},
\end{equation*}
with ${w}_t^{(m)}=\frac{q(\bm{S}_t \mid \bm{S}_{t_2^\ast}^{(m)})}{q(\bm{S}_{t_1^\ast}^{\mathrm{ref}} \mid \bm{S}_{t_2^\ast}^{(m)})}$ and $\widetilde{w}_t^{(m)}=\frac{{w}_t^{(m)}}{\sum_{j=1}^M {w}_t^{(j)}}$,
\begin{equation*}
\widehat{V}_\phi(\bm{S}_t \mid \bm{S}_{t_1^\ast}^{\mathrm{ref}} )
=
\sum_{m=1}^M \widetilde{w}_t^{(m)} \mathcal{R}(\mathcal{T}_{\theta, t_2^\ast:0}^{\mathrm{ref}}(\bm{S}_{t_2^\ast}^{(m)}), \bm{S}^{\mathrm{ref}}).
\end{equation*}
Define
\[
\ell_t^{(m)}=\log q(\bm{S}_t \mid \bm{S}_{t_2^\ast}^{(m)})-\log q(\bm{S}_{t_1^\ast}^{\mathrm{ref}} \mid \bm{S}_{t_2^\ast}^{(m)}),
\]
then,
\begin{equation*}
\widetilde{w}_t^{(m)}=\frac{\exp (\ell_t^{(m)})}{\sum_{j=1}^M \exp (\ell_t^{(j)})}=\operatorname{softmax}_m(\{\ell_t^{(j)}\}_{j=1}^M).
\end{equation*}
Therefore,
\begin{equation*}
    \widehat{V}_\phi(
    \bm{S}_t\mid \bm{S}_{t_1^\ast}^{\mathrm{ref}}
    )
    =
    \sum_{m=1}^{M}
    \operatorname{softmax}_m(\{\ell_t^{(j)}\}_{j=1}^M
    )
    \mathcal R(
    \mathcal T_{\theta,t_2^\ast:0}^{\mathrm{ref}}
    (\bm{S}_{t_2^\ast}^{(m)}),
    \bm{S}^{\mathrm{ref}}
    ).
    \end{equation*}
\end{proof}

\paragraph{Detailed derivations of log transitions}
\label{appx:closeform}

\begin{proof}
The forward transition factorizes over atom positions and atom types,
\begin{equation*}
q(\bm{S}_t \mid \bm{S}_{t_2^\ast})
=
q(\bm{X}_t \mid \bm{X}_{t_2^\ast})
q(\bm{V}_t \mid \bm{V}_{t_2^\ast}),
\end{equation*}
then,
\begin{equation*}
\log q(\bm{S}_t \mid \bm{S}_{t_2^\ast})
=
\log q(\bm{X}_t \mid \bm{X}_{t_2^\ast})
+
\log q(\bm{V}_t \mid \bm{V}_{t_2^\ast}).
\end{equation*}
Similarly,
\begin{equation*}
\log q(
\bm{S}_{t_1^\ast}^{\mathrm{ref}}
\mid
\bm{S}_{t_2^\ast}
)
=
\log q(
\bm{X}_{t_1^\ast}^{\mathrm{ref}}
\mid
\bm{X}_{t_2^\ast}
)
+
\log q(
\bm{V}_{t_1^\ast}^{\mathrm{ref}}
\mid
\bm{V}_{t_2^\ast}
).
\end{equation*}
Therefore,
\begin{equation*}
\ell_t^{(m)}
=
\ell_{t,{\bm{X}}}^{(m)}
+
\ell_{t,{\bm{V}}}^{(m)},
\end{equation*}
where
\begin{equation*}
\ell_{t,{\bm{X}}}^{(m)}
=
\log q(
\bm{X}_t \mid \bm{X}_{t_2^\ast}^{(m)}
)
-
\log q(
\bm{X}_{t_1^\ast}^{\mathrm{ref}}
\mid
\bm{X}_{t_2^\ast}^{(m)}
),
\end{equation*}
and
\begin{equation*}
\ell_{t,{\bm{V}}}^{(m)}
=
\log q(
\bm{V}_t \mid \bm{V}_{t_2^\ast}^{(m)}
)
-
\log q(
\bm{V}_{t_1^\ast}^{\mathrm{ref}}
\mid
\bm{V}_{t_2^\ast}^{(m)}
).
\end{equation*}

For atom positions, the DDPM forward process admits the closed-form transition
\begin{equation*}
q(
\bm{X}_t \mid \bm{X}_{t_2^\ast}
)
=
\mathcal{N}(
\bm{X}_t;
\sqrt{\frac{\bar{\alpha}_t^x}{\bar{\alpha}_{t_2^\ast}^x}}\,\bm{X}_{t_2^\ast},
(
1-\frac{\bar{\alpha}_t^x}{\bar{\alpha}_{t_2^\ast}^x}
)\bm{I}
).
\end{equation*}
Let
\begin{equation*}
\alpha_{t|t_2^\ast}^x
=
\frac{\bar{\alpha}_t^x}{\bar{\alpha}_{t_2^\ast}^x},
\qquad
\sigma_{t|t_2^\ast}^{2}
=
1-\alpha_{t|t_2^\ast}^x.
\end{equation*}
Then, for a scaffold with $n$ atoms,
\begin{equation*}
\log q(
\bm{X}_t \mid \bm{X}_{t_2^\ast}^{(m)}
)
=
-\frac{3n}{2}\log(2\pi\sigma_{t|t_2^\ast}^{2})
-
\frac{
\left\|
\bm{X}_t-\sqrt{\alpha_{t|t_2^\ast}^x}\bm{X}_{t_2^\ast}^{(m)}
\right\|_2^2
}{2\sigma_{t|t_2^\ast}^{2}}.
\end{equation*}

Similarly,
\begin{equation*}
\log q(
\bm{X}_{t_1^\ast}^{\mathrm{ref}}
\mid
\bm{X}_{t_2^\ast}^{(m)}
)
=
-\frac{3n}{2}\log(2\pi\sigma_{t_1^\ast|t_2^\ast}^{2})
-
\frac{
\left\|
\bm{X}_{t_1^\ast}^{\mathrm{ref}}
-\sqrt{\alpha_{t_1^\ast|t_2^\ast}^x}\bm{X}_{t_2^\ast}^{(m)}
\right\|_2^2
}{
2\sigma_{t_1^\ast|t_2^\ast}^{2}
}.
\end{equation*}
Therefore,
\begin{align*}
\ell_{t,\bm{X}}^{(m)}
&=
-\frac{3n}{2}
\log
\frac{
\sigma_{t|t_2^\ast}^{2}
}{
\sigma_{t_1^\ast|t_2^\ast}^{2}
}
-
\frac{
\left\|
\bm{X}_t-\sqrt{\alpha_{t|t_2^\ast}^x}\bm{X}_{t_2^\ast}^{(m)}
\right\|_2^2
}{
2\sigma_{t|t_2^\ast}^{2}
}
\nonumber\\
&\quad+
\frac{
\left\|
\bm{X}_{t_1^\ast}^{\mathrm{ref}}
-\sqrt{\alpha_{t_1^\ast|t_2^\ast}^x}\bm{X}_{t_2^\ast}^{(m)}
\right\|_2^2
}{
2\sigma_{t_1^\ast|t_2^\ast}^{2}
}.
\end{align*}
For atom types, the forward process is
\begin{equation*}
q(
\bm{V}_t \mid \bm{V}_{t_2^\ast}
)
=
\operatorname{Cat}(
\alpha_{t|t_2^\ast}^c
\operatorname{onehot}(\bm{V}_{t_2^\ast})
+
(1-\alpha_{t|t_2^\ast}^c)
\frac{1}{d}\mathbf{1}
),
\end{equation*}
where
\begin{equation*}
\alpha_{t|t_2^\ast}^c
=
\frac{\bar{\alpha}_t^c}{\bar{\alpha}_{t_2^\ast}^c}.
\end{equation*}
For a $n$-atom scaffold,
\begin{equation*}
\log q(
\bm{V}_t \mid \bm{V}_{t_2^\ast}^{(m)}
)
=
\sum_{i=1}^{n}
\log
\left[
\alpha_{t|t_2^\ast}^c
\operatorname{onehot}(V_{t_2^\ast,i}^{(m)})
+
(1-\alpha_{t|t_2^\ast}^c)
\frac{1}{d}\mathbf{1}
\right]_{V_{t,i}},
\end{equation*}
and
\begin{equation*}
\log q(
\bm{V}_{t_1^\ast}^{\mathrm{ref}}
\mid \bm{V}_{t_2^\ast}^{(m)}
)
=
\sum_{i=1}^{n}
\log
\left[
\alpha_{t_1^\ast|t_2^\ast}^c
\operatorname{onehot}(V_{t_2^\ast,i}^{(m)})
+
(1-\alpha_{t_1^\ast|t_2^\ast}^c)
\frac{1}{d}\mathbf{1}
\right]_{V_{t_1^\ast,i}^{\mathrm{ref}}}.
\end{equation*}
Therefore,
\begin{align*}
\ell_{t,\bm{V}}^{(m)}
&=
\sum_{i=1}^{n}
\log
(
\left[
\alpha_{t|t_2^\ast}^c
\operatorname{onehot}(V_{t_2^\ast,i}^{(m)})
+
(1-\alpha_{t|t_2^\ast}^c)
\frac{1}{d}\mathbf{1}
\right]_{V_{t,i}}
)
\\&
\quad-
\sum_{i=1}^{n}
\log
(
\left[
\alpha_{t_1^\ast|t_2^\ast}^c
\operatorname{onehot}(V_{t_2^\ast,i}^{(m)})
+
(1-\alpha_{t_1^\ast|t_2^\ast}^c)
\frac{1}{d}\mathbf{1}
\right]_{V_{t_1^\ast,i}^{\mathrm{ref}}}
).
\end{align*}
\end{proof}

\subsection{Baselines for Experiments}
\label{appx:settings:baselines}

We evaluate three types of baselines on scaffold hopping: 
\textbf{(1)} SBDD models adapted to scaffold hopping (Appendix~\ref{app:finetune_SBDD}), 
\textbf{(2)} a structure-based scaffold hopping model, \diffhopp~\citep{torge_diffhopp} and 
\textbf{(3)} a ligand-based scaffold hopping model, \shepherd~\citep{adams_shepherd}.
For SBDD models,
we choose two recent state-of-the-art
diffusion-based models, \conditar\citep{gao_conditardev} and \ipdiff\citep{huang_ipdiff},
due to their strong generation performance and the availability of public checkpoints for adaptation.
To enable stable trajectory recovery and better preservation of reference shape and interaction with the binding pockets, 
we slightly modify \ipdiff's generation process,
with details provided in Appendix~\ref{appx:ipdiff_modify}.
We denote the adapted models of \conditar and \ipdiff as \ftconditar and \ftipdiff, respectively.
We also include \diffhopp, the only structure-based diffusion model for scaffold hopping
 with publicly available checkpoints.
These three models, \ftconditar, \ftipdiff, and \diffhopp, serve two roles in our experiments: 
they are evaluated directly as scaffold hopping baselines, 
and they are also used as the base models on which \method performs inference-time 
editing.
In addition to these models, we use a ligand-based model, \shepherd, as a baseline.
\shepherd performs scaffold hopping via ligand-based inpainting given 3D shape, functional group, and/or electrostatic conditions.
We do not use \shepherd as a base model for \method because it generates molecules through inpainting,
making both \method and its adaptation strategy not directly applicable.

\subsection{Evaluation Metrics}
\label{appx:settings:metrics}

We evaluate the generated scaffolds from two perspectives: 
\textbf{(1)} similarities with respect to the reference ligands, and 
\textbf{(2)} general physicochemical 
properties for drug development.
Please note, we use similarities between entire molecules as a proxy for scaffold similarity,
as functional groups already remain fixed for a given reference.
For similarities with the reference ligands, we use \simtwod, which is measured by the Tanimoto distance between molecular fingerprints, 
and \simthreed, which is calculated using ROCS tool following {\citet{hawkins2007comparison}}.
To ensure that reference functional-group atoms remain at the same positions in the generated and reference molecules, 
we skip the pre-alignment step in ROCS before \simthreed computation, 
ensuring \simthreed reflects the generated pose.
For general physicochemical properties, 
we report three binding affinity metrics predicted by AutoDock Vina \citep{eberhardt_autodock}:
Vina Score (Vina S) on the original generated pose, Vina Minimization (Vina M) after local energy minimization, and Vina Dock (Vina D) after docking.
We also use quantitative estimates of drug-likeness (QED) and synthetic accessibility (SA) following \citet{luo20213d} as metrics to assess the overall physicochemical profile and 
synthesizability
of generated molecules.
In addition, following \citet{torge_diffhopp}, we evaluate \connectivity,  
defined as the percentage of generated molecules in which the generated scaffold is successfully connected to reference functional groups.
For each reference ligand, we generate 100 samples
and consider only 
the resulting valid and connected molecules for evaluation.
Each metric is first averaged over the generated molecules on a per-reference basis and then across all test complexes.

\subsection{Anchor molecule selection for \methodite}
\label{appx:molecule_selection}
For \methodite, 
we select the anchor molecule from the molecules generated in \hopone using a simple hierarchical filtering procedure. 
For each reference, 
we first retain the top 30\% of molecules according to 3D shape similarity to the reference, and then select the candidate with the lowest 2D similarity as the anchor for \hoptwo.

\subsection{Implementation Details and Reproducibility}
\label{appx:implementation}
\methodonehop uses $\lambda=1$ in its reward for reference-optimal perturbation segment selection for all models.
\methodite follows the same settings, except for \ftipdiff, where $\lambda=0.7$  segment selection.
During value-guided scaffold sampling,
\methodonehop uses $\lambda=0.8$ for its reward,
while \methodite uses $\lambda=0.5$ to better encourage retaining the reference 3D shape over multiple hops.
Both \ftconditar and \ftipdiff use $T=1000$, following their original papers~\citep{gao_conditardev, huang_ipdiff}.
For these models, \methodonehop and \methodite use $N=10$, $n_1=5$ and $n_2=9$.
\diffhopp uses $T=500$, following its original paper~\citep{torge_diffhopp}.
Therefore, \methodonehop and \methodite use $N=5$, $n_1=3$ and $n_2=5$ for this model to maintain the same segment length as other models.
For selecting optimal segment  in~\ref{eq:timestep_score}, we use $K=100$.
For estimating the value in~\ref{eq:value_est_mc}, we use $M=1000$.

\subsection{Modification to \ipdiff's Generation Process}
\label{appx:ipdiff_modify}
\ipdiff extends the standard SBDD diffusion process by incorporating protein-ligand interaction embeddings into both the forward and reverse diffusion processes.
During generation, \ipdiff dynamically extracts the interaction embeddings at each timestep using the predicted clean ligand from the previous timestep,
leading to transition distributions evolving over the denoising trajectory.
To enable a consistent and stable trajectory recovery and preserve the reference 3D shape and pocket-ligand interaction pattern, 
we modify \ipdiff's original inference process by fixing the interaction embeddings to those extracted from the anchor ligand during generation in both hopping iterations.

\subsection{Additional Experimental Results}
\label{appx:experimental_results}

\begin{table}[htbp]
\centering
\caption{
Comparison of generation results with and without trajectory inversion using \ftconditar as the base model.
Random perturbation corresponds to generation under $p_\phi(\cdot\mid t_1^\ast)$, while value-guided sampling further performs search from $t_1^\ast$ to $t=0$.
Best and second-best results are shown in \textbf{bold} and \underline{underlined}, respectively.
}
\label{tab:inversion_comparison}

\begin{threeparttable}
\setlength{\tabcolsep}{4pt}

\begin{tabular*}{\textwidth}{
    @{\extracolsep{\fill}}
    c
    >{\centering\arraybackslash}p{1.6cm}
    l
    r
    r
    r
    r
    r
    r
    r
    r
}
\toprule
\multirow{2}{*}{Inversion}
& \multirow{2}{*}{Sampling}
& \multirow{2}{*}{}
& \multicolumn{2}{c}{Similarity}
& \multicolumn{3}{c}{Vina}
& \multirow{2}{*}{QED $\uparrow$}
& \multirow{2}{*}{SA $\uparrow$}
& \multirow{2}{*}{\shortstack[c]{Conn.\\(\%) $\uparrow$}} \\
\cmidrule(lr){4-5}
\cmidrule(lr){6-8}
&
&
&
\simtwod $\downarrow$
& \simthreed $\uparrow$
& Vina S $\downarrow$
& Vina M $\downarrow$
& Vina D $\downarrow$
&
&
\\
\midrule

\multirow{4}{*}{\cmark}
& \multirow{2}{*}{\shortstack[c]{random\\perturbation}}
& \methodonehop
& $0.391$
& $0.866$
& $\mathbf{-7.837}$
& $\mathbf{-8.308}$
& $\underline{-9.028}$
& $\mathbf{0.462}$
& $\underline{0.629}$
& $92.3$ \\

&
&
\methodite
& $0.356$
& $0.854$
& $-7.763$
& $-8.232$
& $-8.915$
& $0.450$
& $0.620$
& $95.2$ \\

\addlinespace[2pt]
\cmidrule{2-11}
\addlinespace[2pt]

&
\multirow{2}{*}{\shortstack[c]{value-guided\\sampling}}
& \methodonehop
& $0.340$
& $0.882$
& $-7.612$
& $-8.157$
& $-8.904$
& $0.440$
& $0.628$
& $95.1$ \\

&
&
\methodite
& $\underline{0.334}$
& $0.887$
& $-7.539$
& $-8.054$
& $-8.806$
& $0.432$
& $0.621$
& $93.8$ \\

\addlinespace[2pt]
\cmidrule{1-11}
\addlinespace[2pt]

\multirow{4}{*}{\xmark}
& \multirow{2}{*}{\shortstack[c]{random\\perturbation}}
& \methodonehop
& $0.401$
& $0.833$
& $-7.668$
& $-8.144$
& $-8.794$
& $\underline{0.460}$
& $0.623$
& $91.7$ \\

&
&
\methodite
& $0.354$
& $0.845$
& $\underline{-7.825}$
& $\underline{-8.275}$
& $-8.853$
& $0.450$
& $0.622$
& $\mathbf{97.2}$ \\

\addlinespace[2pt]
\cmidrule{2-11}
\addlinespace[2pt]

&
\multirow{2}{*}{\shortstack[c]{value-guided\\sampling}}
& \methodonehop
& $0.353$
& $0.852$
& $-7.497$
& $-7.985$
& $-8.757$
& $0.426$
& $\mathbf{0.634}$
& $\underline{95.9}$ \\

&
&
\methodite
& $0.341$
& $0.884$
& $-7.644$
& $-8.220$
& $\mathbf{-9.053}$
& $0.443$
& $0.624$
& $94.7$ \\

\bottomrule
\end{tabular*}
\end{threeparttable}
\end{table}

\begin{figure}[t]
\begin{center}
\includegraphics[width=\linewidth]{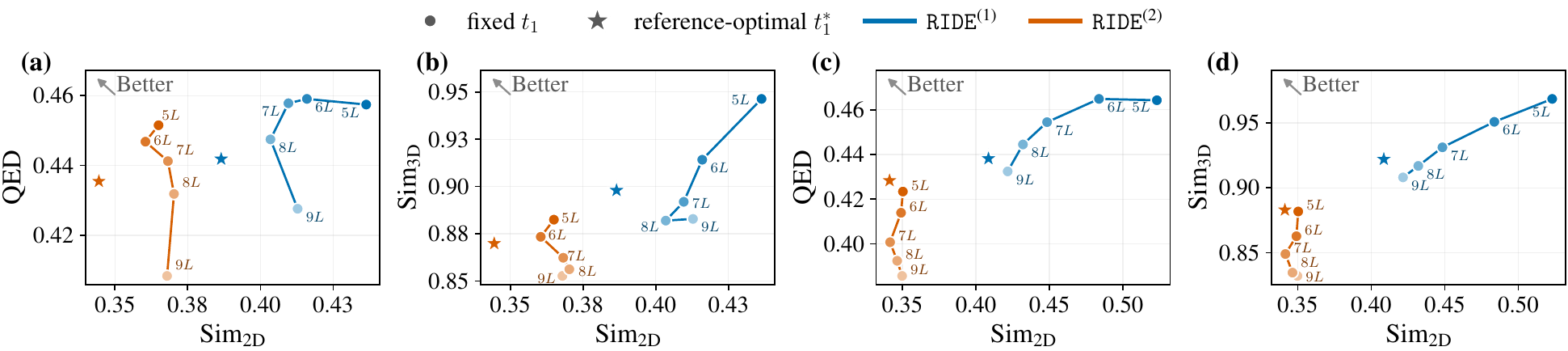}
\caption{
Comparison of Sim$_{\text{2D}}$ vs. QED and Sim$_{\text{2D}}$ vs. Sim$_{\text{3D}}$ results across $t_1$ values ($\{nL\}_{n=5}^{9}$) and  $t_1^\ast$ with $L=T/10$, under reward function of Sim$_{\text{2D}}$ and QED. \textbf{(a)}-\textbf{(b)}: \ftconditar, \textbf{(c)}-\textbf{(d)}: \ftipdiff.
}
\label{fig:qed}
\end{center}
\end{figure}

\begin{figure}[htbp]
\begin{center}
\includegraphics[width=0.5\linewidth]{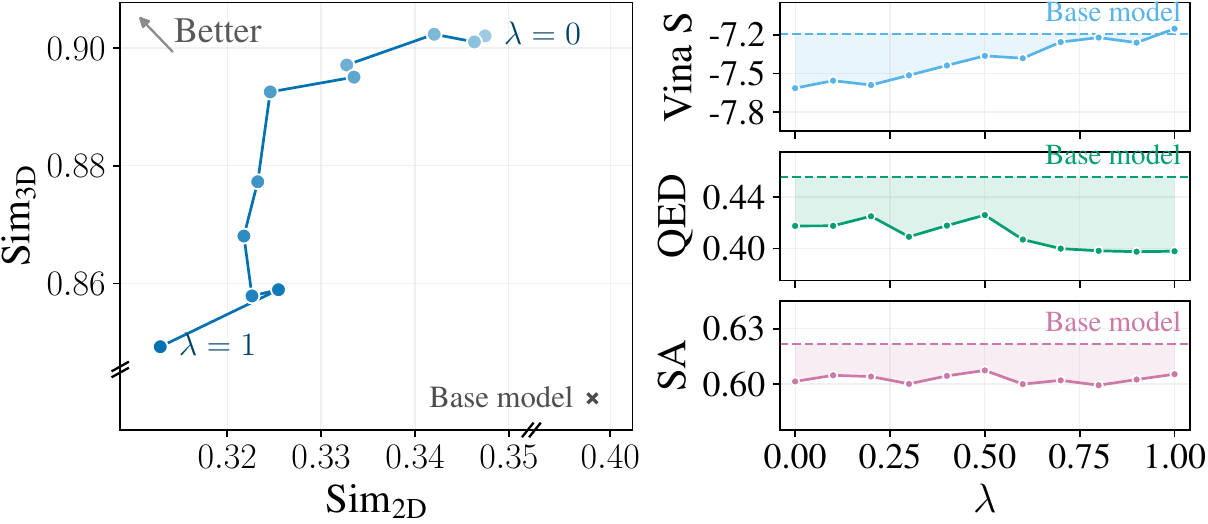}
\caption{
Sensitivity analysis of the reward weight $\lambda$ in the second hop.
}
\label{fig:lambda}
\end{center}
\end{figure}

\paragraph{Comparison of Value-Guided Sampling and Random Perturbation under QED optimization} We observe a similar behavior to that in Figure~\ref{fig:setting_a} when the downstream objective is changed from \simthreed to QED, as shown in Figure~\ref{fig:qed}.
Across candidate $t_1$ values, 
lower \simtwod is accompanied by changes in QED, showing that the perturbation starting timestep also controls the trade-off between the \simtwod and QED objectives.
By comparing the QED curves with the \simthreed curves, 
we observe that QED and \simthreed show similar responses to the perturbation starting timestep, 
with higher QED generally corresponding to higher \simthreed under random perturbation.
This suggests that 3D preservation is partially aligned with good QED.
Our similarity-based dynamic timestep selection also identifies a reasonable starting point balancing \simtwod, QED and \simthreed for subsequent optimization.

\paragraph{Comparison of Inversion and Non-Inversion}
By comparing results with and without inversion in Table~\ref{tab:inversion_comparison},
we observe that removing inversion leads to worse \simthreed and \simtwod, 
highlighting the importance of trajectory inversion in balancing 2D novelty and 3D shape preservation. 
With \methodonehop, trajectory inversion consistently improves both objectives, increasing \simthreed from $0.833$ to $0.866$ at $t_1^\ast$ and from $0.852$ to $0.882$ after search, 
while also reducing \simtwod. 
With \methodite, the improvement is more modest but remains consistent, with \simthreed increasing from $0.845$ to $0.854$ at $t_1^\ast$ and from $0.884$ to $0.887$ after search, 
together with a further improvement in \simtwod after value-guided sampling.

\paragraph{Sensitivity to $\mathbf{\lambda}$}

Figure~\ref{fig:lambda} reports the results of \hoptwo with $\lambda$ varied from $0$ to $1$ in increments of $0.1$.
Across the tested values, increasing $\lambda$ is associated with lower \simtwod and a gradual decrease in \simthreed.
This trend is consistent with the role of $\lambda$ in the reward, where larger values prioritize 2D dissimilarity over 3D similarity.
The other molecular properties, including QED and SA, remain relatively stable across different $\lambda$ values.
Vina S shows moderate variation with better values generally observed at higher \simthreed.
Among the Pareto-optimal subset of the tested $\lambda$ values, $\lambda=0.5$ provides a favorable trade-off between low \simtwod and high \simthreed.

\end{document}